\documentclass[oneside, 11pt]{article}

\usepackage{mystyle}
\usepackage{thm-restate}
\usepackage{placeins}

\newcommand{\E}{\mathbb{E}}

\renewcommand{\P}{\mathbb{P}}

\newcommand{\Normal}{\mathcal{N}}

\newcommand{\Unif}{\mathrm{Unif}}

\newcommand{\indicator}{\mathbf{1}}

\newcommand{\grad}{\nabla}

\newcommand{\R}{\mathbb{R}}
\newcommand{\N}{\mathbb{N}}

\newcommand{\tr}{\mathrm{tr}}

\DeclarePairedDelimiter{\norm}{\lVert}{\rVert}
\DeclarePairedDelimiter{\abs}{\lvert}{\rvert}
\DeclarePairedDelimiter{\inner}{\langle}{\rangle}
\DeclarePairedDelimiter{\set}{\{}{\}}

\newcommand{\mc}{\mathcal}

\makeatletter
\def\fps@figure{tbp}
\def\fps@table{tbp}
\makeatother
\newif\ifexperimentfloats
\AddToHook{cmd/subsection/before}{\ifexperimentfloats\FloatBarrier\fi}
\newenvironment{experimentlayout}{%
    \begingroup
    \experimentfloatstrue
    \setlength{\abovecaptionskip}{6pt}
    \setlength{\belowcaptionskip}{0pt}
    \setlength{\emergencystretch}{1em}
    \suppressfloats[t]
}{%
    \FloatBarrier
    \endgroup
}

\begin{document}

\title{\textbf{Gradient descent dynamics for deep equilibrium models}}
\hypersetup{pdftitle={Gradient descent dynamics for deep equilibrium models}}
\author{Sanjit Dandapanthula\thanks{Carnegie Mellon University, Department of Statistics}\\[.4em]
    \texttt{\href{mailto:sanjitd@cmu.edu}{sanjitd@cmu.edu}}
    \and
    Aaditya Ramdas\thanks{Carnegie Mellon University, Department of Statistics and Machine Learning Department}\\[.4em]
    \texttt{\href{mailto:aramdas@cmu.edu}{aramdas@cmu.edu}}}
\date{\vspace{.5cm} \today}
\maketitle

\begin{abstract}
    Deep equilibrium models (DEQs) have recently emerged as a powerful paradigm for training infinitely deep weight-tied neural networks that achieve state of the art performance across many modern machine learning tasks. Despite their practical success, theoretically understanding the gradient descent dynamics for training DEQs remains an area of active research. In this work, we rigorously study the gradient descent dynamics for DEQs in the simple setting of linear models and single-index models, filling several gaps in the literature. We prove a matrix conservation law for linear DEQs which implies that the parameters remain trapped on spheres along gradient flow and use this property to show that gradient flow remains well-conditioned for all time. We then prove linear convergence of gradient descent to a global minimizer for linear DEQs and deep equilibrium single-index models under appropriate initialization and with a sufficiently small step size. In these simple settings, we also study the effect of common Neumann approximations used to reduce the cost of backpropagation; for linear DEQs, we show that increasing the number of terms in the Neumann approximation can prevent convergence to zero risk, even with exact forward solves. For nonlinear single-index DEQs, we prove linear convergence for fixed Neumann approximations, including Jacobian-free backpropagation under suitable assumptions. Finally, we validate our theoretical findings through experiments.

\end{abstract}

\section{Introduction}

The deep equilibrium model (DEQ) framework introduced in \citet{bai2019deep} has recently emerged as a powerful paradigm for training neural networks that has enjoyed practical success across many modern machine learning tasks, including image generation \citep{pokle2022deep, geng2023one}, inverse problems \citep{gilton2021deep}, graph neural networks \citep{gu2020implicit, liu2022mgnni}, classification and semantic segmentation \citep{bai2020multiscale}, representation learning \citep{huang2021implicit, wang2024lion}, and language modeling \citep{bai2019deep}. DEQs have been shown to have several advantages over traditional deep neural networks, including lower memory consumption during training \citep{geng2021training, fung2022jfb}, the ability to flexibly adapt to different computational budgets at inference time \citep{bai2021neural, liu2025implicit}, and better representational power \citep{wu2024separation, liu2025implicit}.

Instead of learning an explicit relation $y_i = f(x_i)$ from samples $\set{(x_i, y_i)}_{i = 1}^n$, DEQs attempt to learn an implicit relation $y_i = g(x_i, y_i)$ over a parameterized class of functions $\set{g_\theta}_{\theta \in \Theta}$. Given an input $x$, the fixed point $y_\theta(x)$ of $g_\theta(x,\, \cdot)$ can often be interpreted as the limiting point of the sequence $y^{(t + 1)} = g_\theta(x, y^{(t)})$ for any initial injection $y^{(0)}$; hence, we may view DEQs as a single layer which implicitly defines infinitely deep neural networks with input injection and weight tying across layers. In fact, \citet{bai2019deep} showed that stacked DEQs can be represented by a single DEQ on an augmented state. The training of DEQs is typically performed using the implicit function theorem, and inference is performed using a root-finding algorithm to solve for the fixed point $y_\theta(x)$.

Despite their practical success, theoretically understanding the gradient descent dynamics for training DEQs remains an area of active research \citep{kawaguchi2021theory, ling2023global, wu2024separation}. In this work, we initiate a rigorous study of the gradient descent dynamics for DEQs in the simple setting of multivariate linear models (where the target is $f(x) = \Xi x$) and single-index models (where the target is $f(x) = \sigma(\xi^\top x)$ for some activation function $\sigma : \R \to \R$). We concern ourselves with the following questions:
\begin{enumerate}
    \item[Q1:] Due to the implicit nature of DEQs, the Jacobian $I - \partial_y\, g_\theta(x, y_\theta(x))$ must remain invertible for implicit differentiation. Why does the Jacobian remain well-conditioned during implicit training of linear DEQs? As an example, why do linear DEQs not converge to the trivial model $y_\theta(x) = y_\theta(x)$?
    \item[Q2:] Under what conditions does gradient descent converge to the true parameter when training deep equilibrium single-index models?
    \item[Q3:] In practical DEQ training, the Jacobian inverse $(I - \partial_y\, g_\theta(x, y_\theta(x)))^{-1}$ is often approximated using a truncated Neumann series $\sum_{j = 0}^{m - 1} (\partial_y\, g_\theta(x, y_\theta(x)))^j$ to reduce the cost of backpropagation. How does this approximation affect the convergence of gradient descent for linear and single-index DEQs?
\end{enumerate}
Using techniques from dynamical systems theory and previously developed tools for analyzing gradient flow for explicit single-neuron neural networks (e.g., \cite{yehudai2020learning}), we provide a full answer to Q1 and a partial answer to Q2 and Q3 in the settings of well-specified linear and single-index models. Similarly to \citet{yehudai2020learning}, we work directly with the population risk and focus on the gradient flow and gradient descent dynamics for minimizing this risk, with the linear results also applying to empirical squared loss, and leaving statistical generalization guarantees to future work. Our work makes similar assumptions to prior work on gradient flow for single-index models, but notably, our final rates depend on the \emph{amount of nonlinearity} present in the activation function. By analyzing both linear and nonlinear single-index models, we are therefore able to highlight the differences in the training dynamics of DEQs that arise due to the presence of nonlinearity.

First, we summarize our main contributions in \Cref{sec:contributions}. Then in \Cref{sec:related-work}, we discuss related work on DEQs and gradient flow for neural networks and in \Cref{sec:notation}, we introduce notation used throughout the paper. Next, we review the DEQ framework and the gradient flow framework for analyzing gradient descent dynamics in \Cref{sec:background}. In \Cref{sec:gradient-flow-deq}, we present our main results on the behavior of gradient flow and gradient descent for linear DEQs and deep equilibrium single-index models. In \Cref{sec:backward}, we study Neumann approximations for linear and single-index DEQs. Finally, we provide empirical validation of our theory in \Cref{app:experiments} and all proofs are provided in \Cref{app:multivariate-deqs,app:proofs}.

\subsection{Contributions} \label{sec:contributions}

We summarize our main contributions as follows:
\begin{itemize}
    \item For multivariate linear DEQs, we prove a matrix conservation law which keeps the equilibrium well-defined and implies exponential convergence of gradient flow from every invertible initialization. We also prove that gradient descent also converges linearly for sufficiently small step sizes.
    \item We extend our analysis to single-index models with nonlinear activations. Under regularity conditions on the activation and data, gradient flow converges exponentially fast to the target parameter when initialized sufficiently close to it. Gradient descent converges linearly for sufficiently small step sizes. We also quantify how weak nonlinearity slows convergence.
    \item For scalar linear DEQs, we prove convergence for odd Neumann truncations under appropriate initializations. We further provide an explicit example in which adding a second term prevents convergence to zero risk, \emph{despite exact forward solves}, and show that \emph{damping} (a common practical approximation for DEQ training) restores convergence in this example.
    \item For nonlinear single-index DEQs with exact forward equilibria, we find that every fixed Neumann expansion, including Jacobian-free backpropagation (JFB), converges linearly under suitable initialization and step sizes. Furthermore, under appropriate assumptions, we prove that these updates have the same linearized dynamics as exact backpropagation.
\end{itemize}

\subsection{Related work} \label{sec:related-work}

In this section, we briefly review related work on DEQs and gradient flow for neural networks.

\paragraph{Deep equilibrium models.} DEQs were introduced in \citet{bai2019deep} as a framework for training infinitely deep neural networks with weight tying across layers and input injection. Since then, DEQs have been successfully applied to image generation \citep{pokle2022deep, geng2023one}, inverse problems \citep{gilton2021deep}, graph neural networks \citep{gu2020implicit, liu2022mgnni}, classification and semantic segmentation \citep{bai2020multiscale}, representation learning \citep{huang2021implicit, wang2024lion}, and language modeling \citep{bai2019deep}. Related recurrent-depth language models also scale inference-time computation without solving for an equilibrium \citep{geiping2025scaling}. \citet{mccallum2025reversible} recently suggest that it can be computationally beneficial to create an algebraically reversible fixed-point iteration, so that we can directly backpropagate through the fixed point iteration instead of implicitly differentiating the equilibrium as in \eqref{eq:implicit-jacobian}; their gradients concern the finite iteration rather than the exact equilibrium studied here.

\paragraph{Gradient flow for neural networks.} There is a large body of work establishing that perturbed gradient descent can escape saddle points in non-convex problems which satisfy a \emph{strict saddle property}, meaning that all suboptimal saddle points admit a direction of negative curvature \citep{ge2015escaping, jin2017escape}; see also \citet{sun2015nonconvex} for a discussion of this geometry. For example, this phenomenon has been heavily studied for the phase retrieval problem (with $\sigma(z) = z^2$), where gradient descent is known to converge exponentially fast with appropriate spectral initialization \citep{candes2015phase}; \citet{sun2018geometric} instead analyze the landscape and a trust-region method without special initialization. Gradient methods for nonlinear regression and \emph{explicit} single-index models have been extensively studied in the literature \citep{mei2018landscape, oymak2019overparameterized, yehudai2020learning}, where typical assumptions include monotonicity of the activation function (which is assumed to have a strong gradient near zero) and that the data is sufficiently spread out near zero.

\paragraph{Gradient flow for DEQs.} \citet{kawaguchi2021theory} studies the optimization landscape of implicit models, showing global linear-rate convergence for linear implicit models without a width assumption. However, their results consider a stylized setting including a softmax nonlinearity on the weights in order to ensure well-posedness of the implicit model during training; we make no such modifications. In a similar vein, \citet{ling2023global} and \citet{gao2022optimization} prove global convergence for overparameterized DEQs under suitable initialization assumptions. \citet{truong2025global} extends the results of \citet{ling2023global} to DEQs with general activations having bounded first and second derivatives, showing that global convergence at a linear rate still holds in the overparameterized regime. \citet{gao2023wide} shows that infinitely wide implicit models can be modeled by Gaussian processes, and \citet{feng2023neural} shows that there is a deterministic neural tangent kernel (NTK) for implicit models under mild conditions and show how to find it explicitly by root-finding.

Closest to our work, \citet{wu2024separation} study diagonal linear DEQs with fixed input injection and readout. They characterize the limiting predictor assuming convergence and prove convergence under initialization and design assumptions; their GD result also assumes bounded predictor iterates. We instead train both input and feedback matrices and prove convergence from every invertible initialization.

\paragraph{Approximate backpropagation.} \citet{liao2018reviving} study Neumann recurrent backpropagation. \citet{geng2021training} analyze phantom gradients based on damped unrolling and Neumann approximations, while \citet{fung2022jfb} give conditions for sample-wise descent using JFB. \citet{geng2021training} also observe that longer Neumann expansions can degrade performance due to numerical explosion of higher-order terms. By contrast, we give an explicit scalar linear example in which even just adding a second term prevents convergence to zero risk, despite \emph{exact forward solves}. \citet{bolte2023one} show agreement between one-step and implicit differentiation when the state Jacobian vanishes at the solution. We study convergence of the corresponding population updates for single-index DEQs.

\paragraph{Stability of training DEQs.} There is a long line of work studying well-posedness of the Jacobian inverse in \eqref{eq:implicit-jacobian}. For example, \citet{winston2020monotone}, \citet{revay2020lipschitz}, and \citet{el2021implicit} study sufficient conditions on the form of the neural network $g_\theta$ to ensure that the fixed-point problem is well-posed. \citet{bai2021stabilizing} show that Jacobian regularization improves the stability and convergence of the forward and backward solves. \citet{agarwala2022deep} show that orthogonal or symmetric initialization can improve training stability and permit a broader range of initial weight scales.

\subsection{Notation} \label{sec:notation}

For an $\R^d$-valued random variable $X$, we write $X \in L^p(\P)$ if $\E[\norm{X}_2^p] < \infty$. We write
\begin{align*}
    \norm{f}_\mathrm{Lip} \coloneq \sup_{x \neq y}\, \frac{\norm{f(x) - f(y)}_2}{\norm{x - y}_2}
\end{align*}
for the Lipschitz norm of a function $f : \R^{d_1} \to \R^{d_2}$. We denote by $\grad_\theta$ the gradient operator with respect to the parameter $\theta$. We write $A \succeq 0$ if $A$ is positive semidefinite and $A \succ 0$ if $A$ is positive definite. We write $\lambda_\mathrm{min}(A)$ for the minimum eigenvalue of a symmetric matrix $A \in \R^{d \times d}$. We write $I_d$ for the identity matrix over $\R^d$. We use $\N_0$ to denote the set of nonnegative integers.

\section{Background} \label{sec:background}

In this section, we briefly review deep equilibrium models (DEQs) and the gradient flow framework for analyzing gradient descent dynamics. Suppose that the data is generated according to a function $f : \R^d \to \R$; our goal is to learn this function from samples. In the DEQ framework, we parameterize a class of functions $\set{g_\theta}_{\theta \in \Theta}$ and define the model output $y_\theta(x)$ for an input $x \in \R^d$ as the fixed point of the equation $y = g_\theta(x, y)$.

For the forward pass, we need to solve the fixed-point equation $y_\theta(x) = g_\theta(x, y_\theta(x))$ for $y_\theta(x)$, so we can use any black-box root-finding algorithm or fixed-point iteration. In particular, as long as $g_\theta$ is contractive in its second argument (i.e., $\norm{g_\theta(x,\, \cdot)}_\mathrm{Lip} < 1$), the Banach fixed-point theorem guarantees that the fixed point exists and is unique. Furthermore, we may use simple fixed-point iteration $y^{(t + 1)} = g_\theta(x, y^{(t)})$ to solve for $y_\theta(x)$ starting from any initial choice $y^{(0)}$. The fact that there are no assumptions about the root-finding algorithm allows for \emph{test-time scaling}, since using more iterations may give a more accurate approximation to the equilibrium. Our theory assumes exact forward equilibria.

For the backward pass, we would like to minimize the expected loss by gradient descent. Given a loss function $\ell : \R \times \R \to \R$, we may define the population risk as
\begin{align*}
    R(\theta) \coloneq \E[\ell(y_\theta(X), f(X))],
\end{align*}
where $X$ is an $\R^d$-valued random variable drawn according to the input data distribution. Although in practice we may only have access to finite samples, we focus on the problem of directly minimizing the population risk over $\theta \in \Theta$ using gradient descent. For gradient descent with step size $\eta > 0$ and initialization $\theta \in \Theta$, the parameter updates are given by
\begin{align*}
    \theta(t + 1) = \theta(t) - \eta\, \grad_\theta\, R(\theta(t)).
\end{align*}
To analyze the gradient descent dynamics, we may also consider the continuous-time limit of gradient descent, often referred to as gradient flow. In this framework, we consider the ordinary differential equation (ODE)
\begin{align*}
    \theta^\prime(t) = -\grad_\theta\, R(\theta(t))
\end{align*}
with initialization $\theta(0) = \theta$. Assuming that we can differentiate under the integral sign, it suffices to compute $\grad_\theta\, \ell(y_\theta(x), f(x))$ for a fixed $x$ to obtain the gradient of the population risk. If $(\theta, x, y) \mapsto g_\theta(x, y)$ is $C^1$ near an existing fixed point and $1 - \partial_y g_\theta(x, y_\theta(x)) \neq 0$, the implicit function theorem gives a locally unique differentiable branch of that fixed point, with
\begin{align} \label{eq:implicit-jacobian}
    \nabla_\theta\, y_\theta(x) = \left( 1 - \frac{\partial}{\partial y}\, g_\theta(x, y_\theta(x)) \right)^{-1} \partial_\theta g_\theta(x, y_\theta(x)).
\end{align}
In particular, we can then differentiate both sides of $y_\theta(x) = g_\theta(x, y_\theta(x))$ using the chain rule to get
\begin{align*}
    \nabla_\theta\, \ell(y_\theta(x), f(x))
     & = \left( \frac{\partial}{\partial y}\, \ell(y_\theta(x), f(x)) \right) \nabla_\theta\, y_\theta(x)                                                                                                      \\
     & = \left( \frac{\partial}{\partial y}\, \ell(y_\theta(x), f(x)) \right) \left( 1 - \frac{\partial}{\partial y}\, g_\theta(x, y_\theta(x)) \right)^{-1}\; \partial_\theta g_\theta(x, y_\theta(x)).
\end{align*}
Hence, the Jacobian term $1 - \frac{\partial}{\partial y}\, g_\theta(x, y_\theta(x))$ plays a crucial role in determining the gradient of the loss with respect to the parameters. If this term becomes close to zero during training, the gradient may become unbounded, leading to instability in the training dynamics. Therefore, it is important to understand whether this Jacobian term remains well-conditioned during implicit training of DEQs. In practice, the Jacobian inverse is often approximated using a truncated \emph{Neumann series} $\sum_{j = 0}^{m-1} (\partial_y\, g_\theta(x, y_\theta(x)))^j$ to reduce the cost of backpropagation; we also study the effect of this approximation on the convergence of gradient descent for linear and single-index DEQs in \Cref{sec:backward}. With $m = 1$, this corresponds to \emph{Jacobian-free backpropagation} (JFB) \citep{fung2022jfb}, which has been shown to be effective in practice.

\section{Gradient flow for deep equilibrium models} \label{sec:gradient-flow-deq}

In this section, we discuss our main results about the behavior of gradient flow and gradient descent for DEQs.

\subsection{Linear models} \label{sec:matrix-extension}

First, we consider the linear function $f(x) = \Xi x$ for some fixed $\Xi \in \R^{m \times d}$ and fit a multivariate linear deep equilibrium model
\begin{equation} \label{eq:multivariate-linear-implicit-model}
    y_\Gamma(x) = g_\Gamma(x, y_\Gamma(x)) \coloneq \Gamma_1\, x + \Gamma_2\, y_\Gamma(x),
\end{equation}
where $\Gamma_1 \in \R^{m \times d}$ and $\Gamma_2 \in \R^{m \times m}$. If $I_m - \Gamma_2$ is invertible, we can solve explicitly for the equilibrium:
\begin{align*}
    y_\Gamma(x) = (I_m - \Gamma_2)^{-1}\, \Gamma_1\, x.
\end{align*}
Let $X \in L^2(\P; \R^d)$ and $\Sigma = \E[XX^\top]$. We use the squared risk
\begin{align*}
    R(\Gamma, \Xi) = \E[\norm{y_\Gamma(X) - \Xi X}_2^2].
\end{align*}

\subsubsection{Gradient flow}

We begin by proving a conservation law for the dynamics of gradient flow, showing that the parameters remain trapped on Frobenius norm spheres.

\begin{restatable}[Matrix conservation law]{theorem}{matrixconservation} \label{thm:matrix-conservation}
    Suppose $I_m - \Gamma_2(0)$ is invertible. Gradient flow satisfies the matrix conservation law
    \begin{equation} \label{eq:matrix-conservation}
        \begin{split}
             & (I_m - \Gamma_2(t))\, (I_m - \Gamma_2(t))^\top + \Gamma_1(t)\, \Gamma_1(t)^\top          \\
             & \qquad = (I_m - \Gamma_2(0))\, (I_m - \Gamma_2(0))^\top + \Gamma_1(0)\, \Gamma_1(0)^\top
        \end{split}
    \end{equation}
    for all $t \geq 0$ for which the flow is well-defined.
\end{restatable}

The proof (\Cref{app:proof-matrix-conservation}) follows from an explicit computation of the gradient flow for this problem. Taking the trace on both sides of \eqref{eq:matrix-conservation} recovers conservation of $\norm{\Gamma_1}_F^2 + \norm{I_m - \Gamma_2}_F^2$, which can also be recovered from the Lie-theoretic framework of \citet{marcotte2023abide}. Intuitively, this conservation law arises due to the rescaling symmetry $(\Gamma_1, \Gamma_2) \mapsto (\alpha \Gamma_1, I_m - \alpha(I_m - \Gamma_2))$ for $\alpha \neq 0$. Next, we show that the gradient flow exists for all time and converges to a global minimizer of the risk.

\begin{restatable}[Global convergence of gradient flow for linear DEQs]{theorem}{matrixflow} \label{thm:matrix-flow}
    Suppose $I_m - \Gamma_2(0)$ is invertible. Gradient flow exists for all $t \geq 0$ and, for some $c > 0$, satisfies
    \begin{align*}
        R(\Gamma(t), \Xi) \leq R(\Gamma(0), \Xi)\, e^{-ct}.
    \end{align*}
    The parameters converge to a limit $\Gamma^\infty$ with $I_m - \Gamma_2^\infty$ invertible and $y_{\Gamma^\infty}(X) = \Xi X$ almost surely. If in addition $\Sigma \succ 0$, then $(I_m - \Gamma_2^\infty)^{-1}\, \Gamma_1^\infty = \Xi$.
\end{restatable}

The proof (\Cref{app:proof-matrix-flow}) combines the conservation law with the fact that the risk decreases to keep $I_m - \Gamma_2$ away from singularity. Then, the gradient flow is shown to converge to a global minimizer of the risk using Gr\"onwall's inequality.

\subsubsection{Gradient descent}

Next, under the linear model, we show that gradient descent also converges to a global minimizer of the risk for sufficiently small step sizes.

\begin{restatable}[Global convergence of gradient descent for linear DEQs]{theorem}{matrixdescent} \label{thm:matrix-descent}
    Suppose $I_m - \Gamma_2(0)$ is invertible. There is $\eta_0 > 0$, depending on the initialization, target, and data covariance, such that exact gradient descent with any fixed step size $0 < \eta \leq \eta_0$ satisfies
    \begin{align*}
        \inf_{t \in \N_0} \sigma_\mathrm{min}(I_m - \Gamma_2(t)) > 0, \qquad
        R(\Gamma(t), \Xi) \leq (1 - c)^t\, R(\Gamma(0), \Xi)
    \end{align*}
    for some $c \in (0, 1)$. The parameters converge to a limit $\Gamma^\infty$ with $I_m - \Gamma_2^\infty$ invertible and $y_{\Gamma^\infty}(X) = \Xi X$ almost surely.
\end{restatable}

The proof and an explicit step-size bound are given in \Cref{app:proof-matrix-descent}. The proof proceeds by first establishing an approximate conservation law and proving that the gradient descent iterates remain trapped in a compact set for all time. Then, we prove that the risk satisfies a Polyak-\L{}ojasiewicz (P\L{}) inequality within this set and thereby conclude the convergence of gradient descent.

Lastly, we extend our results to the case of general labels $Y \in L^2(\P; \R^m)$, which may not be linear in $X$. Here, we show that the gradient flow and gradient descent dynamics for the linear DEQ converge to the best linear predictor of $Y$ from $X$. To prove this result, we decompose the squared loss into the error of the best linear predictor and the excess risk of the DEQ. Since the first term is constant in the parameters, the gradient flow and gradient descent dynamics are unchanged. The proof is in \Cref{app:proof-matrix-general-labels}.

\begin{restatable}[General labels]{corollary}{matrixgenerallabels} \label{cor:matrix-general-labels}
    Let $Y \in L^2(\P; \R^m)$ be arbitrary. The conclusions of \Cref{thm:matrix-flow,thm:matrix-descent} hold for $\mc{L}(\Gamma) = \E[\norm{y_\Gamma(X) - Y}_2^2]$, after replacing $R$ throughout by $\mc{L} - \inf_{\Psi \in \R^{m \times d}} \E[\norm{\Psi X - Y}_2^2]$ and defining
    \begin{align*}
        \Xi = \E[YX^\top]\, \Sigma^{\dagger}.
    \end{align*}
    In particular, the results apply to finite empirical squared loss with arbitrary labels, irrespective of whether the empirical covariance is full rank. The limit attains the best linear prediction risk, which need not be zero.
\end{restatable}

We include experiments for scalar linear DEQs in \Cref{app:linear-experiments}, showing that gradient descent learns the target even when initialized close to the singular parameter $(0, 1)$.

\subsection{Nonlinear single-index models}

Next, we consider the implicit target $y_\xi(x) = \sigma(\xi_1^\top x + \xi_2\, y_\xi(x))$ for some fixed nonzero $\xi = (\xi_1, \xi_2) \in \R^{d + 1}$ and nonlinear activation function $\sigma : \R \to \R$. For $y_\xi(x)$ to be well-defined, we assume that $\abs{\xi_2} < 1 / L$ where $L \coloneq \norm{\sigma}_\mathrm{Lip}$, so that the target fixed-point equation is a contraction and $y_\xi(x)$ exists and is unique by the Banach fixed-point theorem. Note that $\xi_2 = 0$ recovers the explicit single-neuron target $y_\xi(x) = \sigma(\xi_1^\top x)$.

\subsubsection{Gradient flow}

In this section, we analyze the gradient flow dynamics for fitting this target using a single implicit neuron
\begin{equation} \label{eq:implicit-neuron}
    y_\theta(x) = g_\theta(x, y_\theta(x)) \coloneq \sigma(\theta_1^\top x + \theta_2\, y_\theta(x)).
\end{equation}
Note that for $y_\theta(x)$ to be well-defined, it suffices to have $\sigma$ differentiable with $\norm{\sigma}_\mathrm{Lip} = L$ and $\abs{\theta_2} \leq \delta_2 < 1 / L$. This way, we have the guarantee
\begin{align*}
    \abs*{\frac{\partial g_\theta(x, y)}{\partial y}} = \abs{\theta_2}\, \abs{\sigma^\prime(\theta_1^\top x + \theta_2\, y)} \leq L \delta_2 < 1,
\end{align*}
meaning that $g_\theta(x, \cdot)$ is a global contraction mapping for each $x \in \R^d$ and the Banach fixed-point theorem implies that there exists a unique fixed-point $y_\theta(x)$ satisfying the implicit equation above. We assume that the input data $X$ is an $\R^d$-valued random variable and use the squared loss $\ell(y, \hat{y}) = (\hat{y} - y)^2$; our goal is now to minimize the population risk
\begin{align*}
    R(\theta, \xi) \coloneq \E[(y_\theta(X) - y_\xi(X))^2].
\end{align*}
Notice that $\xi = (\xi_1, \xi_2)$ is a global minimizer of the population risk because the model is well-specified. Now, for many activation functions and data distributions, we show that gradient flow initialized close enough to the target parameter $\xi$ converges exponentially fast to this target.

\begin{restatable}[Gradient flow for single-index DEQs, simplified]{theorem}{nonlinearflowsimplified} \label{thm:nonlinear-convergence-simplified}
    Assume that
    \begin{enumerate}
        \item The data satisfies $X \in L^2(\P)$ and has a positive density on a nonempty open ball.
        \item The activation $\sigma \in C^1(\R)$ is globally $L$-Lipschitz, satisfies $\sigma^\prime(z) > 0$ for every $z$, and is not affine on any nonempty open interval.
        \item The target satisfies $\abs{\xi_2} < 1 / L$ and either $\xi_1 \neq 0$ or $\sigma(0) \neq 0$.
    \end{enumerate}
    For any initialization in $B(\xi, 1 / L - \abs{\xi_2})$, gradient flow exists for all $t \geq 0$ and satisfies
    \begin{align*}
        \norm{\theta(t) - \xi}_2^2 \leq \norm{\theta(0) - \xi}_2^2\, e^{-\lambda t}
    \end{align*}
    for some $\lambda > 0$ depending on the target, initialization, activation, and data distribution.
\end{restatable}

The more general statement is given with precise assumptions in \Cref{thm:nonlinear-convergence} in \Cref{app:general-nonlinear-convergence}. For instance, the theorem applies to sigmoid and softplus, and to hyperbolic tangent when $\xi_1 \neq 0$. We first use the assumptions on the activation and data distribution to show that the target parameter is identifiable. Then, we prove that the gradient has a sufficiently large component in the direction $\theta - \xi$, which implies that the distance to the target decreases exponentially fast. The proof is in \Cref{app:proof-cor-nonlinear-convergence}.

\subsubsection{Gradient descent}

By a similar analysis, we conclude that gradient descent with sufficiently small step size also converges to the target parameter at a linear rate when initialized close to the target.

\begin{restatable}[Gradient descent for single-index DEQs, simplified]{theorem}{nonlineardescentsimplified} \label{thm:nonlinear-gd-convergence}
    Under the assumptions and initialization of \Cref{thm:nonlinear-convergence-simplified}, there are constants $\lambda, \eta_0 > 0$ such that, for every $0 < \eta \leq \eta_0$, gradient descent satisfies
    \begin{align*}
        \norm{\theta(t) - \xi}_2^2 \leq (1 - \eta\, \lambda)^t\, \norm{\theta(0) - \xi}_2^2, \qquad t \in \N_0,
    \end{align*}
    with $0 < 1 - \eta\, \lambda < 1$. The constants may depend on the target, initialization, activation, and data distribution.
\end{restatable}

The general statement is given in \Cref{thm:nonlinear-gd-convergence-general} in \Cref{app:general-nonlinear-gd}. The proof uses the same bound on the component of the gradient in the direction $\theta - \xi$, together with an upper bound on the gradient norm. Choosing the step size sufficiently small then ensures that the squared distance to the target decreases by a fixed factor at each iteration. The proof is in \Cref{app:proof-nonlinear-gd-simplified}. Crucially, the convergence rate $\lambda$ that we obtain depends on the amount of nonlinearity in the activation function, as we discuss more formally in \Cref{rem:nonlinearity} below.

\begin{restatable}[Amount of nonlinearity affects convergence rate]{remark}{nonlinearityremark} \label{rem:nonlinearity}
    Consider an explicit target $\xi = (\xi_1, 0)$ and $\sigma_\varepsilon(z) = z + \varepsilon h(z)$, where $h \in C^1(\R)$ has bounded derivative. If $X \in L^2(\P)$ and
    \begin{align*}
        \E \left[
        \begin{pmatrix}
            X \\ h(\xi_1^\top X)
        \end{pmatrix}
        \begin{pmatrix}
            X \\ h(\xi_1^\top X)
        \end{pmatrix}^{\!\top} \right] \succ 0,
    \end{align*}
    then, as $\varepsilon \to 0$ through nonzero values,
    \begin{align*}
        \lambda_\mathrm{min} \bigl(\grad_\theta^2\, R_\varepsilon(\xi, \xi) \bigr) \asymp \varepsilon^2.
    \end{align*}
    Note that at $\varepsilon = 0$, the direction $(-\xi_1, 1)$ is redundant. Along this direction, the derivative of the prediction is of order $\varepsilon$. Since the Hessian at a realizable target is a Gram matrix of prediction derivatives, the curvature in this direction is of order $\varepsilon^2$; the Gram condition gives a matching lower bound on the smallest eigenvalue. The slowest local rate of convergence for gradient flow therefore has this order; for a fixed sufficiently small step size $\eta$, the local error contracts by a factor of $1 - \eta\, \Theta(\varepsilon^2)$ per iteration. The proof and explicit constants are in \Cref{app:proof-nonlinearity}.
\end{restatable}

We include experiments illustrating convergence to the target parameter in \Cref{app:nonlinear-experiments}. In \Cref{app:weak-experiment}, we vary the amount of nonlinearity and observe slower parameter recovery for smaller $\varepsilon$, consistent with the $\varepsilon^2$ scaling in \Cref{rem:nonlinearity}.

\section{Neumann approximations and Jacobian-free backpropagation} \label{sec:backward}

In this section, we replace the inverse in \eqref{eq:implicit-jacobian} by a fixed Neumann expansion \citep{liao2018reviving,geng2021training}, including the one-term approximation used by Jacobian-free backpropagation (JFB) \citep{fung2022jfb}. We analyze the resulting training dynamics for linear and nonlinear single-index models. Throughout this section, we assume that the forward equilibrium is computed exactly. For a positive integer $m$, let $p_m(a) \coloneq \sum_{j = 0}^{m - 1} a^j$; note that the polynomial $p_m(a)$ is defined even outside $\abs{a} < 1$, where the infinite Neumann series need not converge. For the single-index model $y_\theta(x) = \sigma(\theta_1^\top x + \theta_2\, y_\theta(x))$ with target $f$, the population update with a Neumann approximation is:
\begin{equation} \label{eq:filtered-field}
    \begin{split}
        G_m(\theta) \coloneq 2\, \E\Biggl[ & (y_\theta(X) - f(X))\, \sigma^\prime(\theta_1^\top X + \theta_2\, y_\theta(X))                                  \\
                                           & \quad p_m\bigl(\theta_2\, \sigma^\prime(\theta_1^\top X + \theta_2\, y_\theta(X))\bigr)\, \begin{pmatrix}
                                                                                                                                               X \\ y_\theta(X)
                                                                                                                                           \end{pmatrix}\Biggr].
    \end{split}
\end{equation}

\subsection{Linear models} \label{sec:linear-truncated}

First, we consider a scalar state with linear activation:
\begin{align*}
    y_\theta(x) = \theta_1^\top x + \theta_2\, y_\theta(x) = \frac{\theta_1^\top x}{1 - \theta_2},
\end{align*}
where the target is $f(x) = \xi^\top x$ with $\xi \in \R^d$. We further assume that $\Sigma = \E[XX^\top] \succ 0$.

\begin{restatable}[Backward approximations rescale the gradient]{lemma}{linearrescaling} \label{lem:linear-time-change}
    For $\theta_2 \neq 1$, the $m$-term backward direction satisfies
    \begin{align*}
        G_m(\theta) = (1 - \theta_2^m)\, \grad_\theta\, R(\theta, \xi).
    \end{align*}
    Along $\theta^\prime = -G_m(\theta)$, we have the conservation law
    \begin{align*}
        w(t) \coloneq \norm{\theta_1(t)}_2^2 + (1 - \theta_2(t))^2 = w(0).
    \end{align*}
\end{restatable}

At a nonzero gradient, a positive factor $1 - \theta_2^m$ gives descent, a zero factor stops the update, and a negative factor gives ascent. The proof is in \Cref{app:proof-linear-rescaling}.

\begin{restatable}[Gradient flow with Neumann approximations]{theorem}{lineartruncflow} \label{thm:linear-trunc-flow}
    Suppose $\theta_2(0) < 1$ and either $m$ is odd, or $m$ is even and $w(0) < 4$. The flow exists for all $t \geq 0$, remains in $\theta_2 < 1$, and satisfies
    \begin{align*}
        R(\theta(t), \xi) \leq R(\theta(0), \xi)\, \exp\left(-\frac{4\, \lambda_\mathrm{min}(\Sigma)\, c_m}{w(0)}\, t\right),
    \end{align*}
    where $c_m > 0$ is a uniform lower bound on the gradient multiplier $1 - \theta_2(t)^m$ along the flow.
    The parameters converge to $(\alpha \xi, 1 - \alpha)$, where $\alpha = \sqrt{w(0) / (1 + \norm{\xi}_2^2)}$.
\end{restatable}

Since the approximate update is a positive scalar multiple of the exact gradient, the two flows follow the same trajectories at different speeds. Using the conservation law and decreasing risk, we show that this multiplier remains bounded away from zero and thereby conclude exponential convergence. The proof and an explicit choice of $c_m$ are in \Cref{app:proof-linear-trunc}. Odd truncations converge even when the forward map is not contractive. For even truncations, the assumption $w(0) < 4$ prevents the trajectory from reaching $\theta_2 = -1$, where the update vanishes; without this assumption, convergence can fail.

\begin{restatable}[Gradient descent with Neumann approximations]{theorem}{lineartruncdescent} \label{thm:linear-trunc-descent}
    Under the same assumptions as \Cref{thm:linear-trunc-flow}, there is $\eta_0 > 0$ such that, for every fixed step size $0 < \eta \leq \eta_0$, the updates $\theta(t + 1) = \theta(t) - \eta\, G_m(\theta(t))$ remain in $\theta_2 < 1$ and satisfy
    \begin{align*}
        R(\theta(t), \xi) \leq (1 - c)^t\, R(\theta(0), \xi), \qquad t \in \N_0,
    \end{align*}
    for some $c \in (0, 1)$. The parameters converge to a zero-risk limit.
\end{restatable}

As in the case of exact backpropagation, we first establish an approximate conservation law which keeps the iterates in a compact set. We then show that the risk satisfies a P\L{} inequality and that the Neumann multiplier is positive and bounded away from zero within this set, implying linear convergence for sufficiently small step sizes. The proof is in \Cref{app:proof-linear-trunc-descent}. We further demonstrate below that an even truncation can fail to converge if $w(0) \geq 4$, even with just two terms in the backward expansion.

\begin{restatable}[A second term can prevent convergence]{proposition}{evenstall} \label{prop:even-stall}
    Let $d = 1$, $\E[X^2] = 1$, and $f(x) = x$. Initialize $\theta(0) = (\sqrt{8}, 0)$. For every even $m \geq 2$, the flow using $G_m$ satisfies
    \begin{align*}
        \theta_2(t) \to -1, \qquad \theta_1(t) \to \sqrt{5}, \qquad R(\theta(t), 1) \to \left(\frac{\sqrt{5}}{2} - 1 \right)^2 > 0
    \end{align*}
    as $t \to \infty$. For $m = 2$, the same failure holds for every fixed step size
    \begin{align*}
        0 < \eta < \min \set*{ \frac{1}{2\, \sqrt{8}\, (\sqrt{8} - 1)}, \frac{1}{2\, (1 + \sqrt{8})} }:
    \end{align*}
    the iterates satisfy $\theta_2(t) \to -1$ and their limiting risk is at least $(\sqrt{5} / 2 - 1)^2$. In contrast, JFB converges to zero risk from this initialization, both in the continuous time limit and for sufficiently small fixed steps.
\end{restatable}

This failure occurs because the two-term backward expansion can vanish at a nonzero gradient, while JFB does not. To gain intuition for this result, we can compare the two updates:
\begin{align*}
    G_1 = (1 - \theta_2)\, \grad_\theta\, R, \qquad G_2 = (1 - \theta_2^2)\, \grad_\theta\, R.
\end{align*}
At $\theta_2 = -1$, the two-term update vanishes while JFB equals twice the exact gradient. For this initialization, the conservation law keeps the prediction error bounded away from zero while $\theta_2$ decreases toward $-1$. Hence, the approximate update tends to zero while the risk remains positive. The equilibrium remains nonsingular, since $1 - \theta_2(t) \to 2$, although the forward contraction factor $\abs{\theta_2(t)}$ approaches one. We include the full proof in \Cref{app:proof-even-stall}.

Next, we study the effect of \emph{damping} the backward expansion, which is a technique introduced by \citet{geng2021training} to improve the stability of training DEQs. Damping replaces the fixed-point map by
\begin{align*}
    g_{\theta, \lambda}(x, y) = (1 - \lambda)\, y + \lambda\, g_\theta(x, y), \qquad \lambda > 0.
\end{align*}
Its fixed points are unchanged, and its $m$-term backward rule with the Neumann approximation is
\begin{align*}
    G_{m, \lambda}(\theta) = \set{1 - (1 - \lambda + \lambda\, \theta_2)^m}\, \grad_\theta\, R(\theta, \xi).
\end{align*}

\begin{restatable}[Convergence with damped Neumann approximations]{corollary}{lineardamped} \label{cor:linear-damped}
    Suppose $\theta_2(0) < 1$ and $\lambda \sqrt{w(0)} < 2$. For each $m \in \N$, there are constants $c, \eta_0 > 0$ such that the flow $\theta^\prime = -G_{m, \lambda}(\theta)$ exists for all $t \geq 0$ and satisfies
    \begin{align*}
        R(\theta(t), \xi) \leq R(\theta(0), \xi)\, e^{-ct}, \qquad t \geq 0.
    \end{align*}
    For every fixed step size $0 < \eta \leq \eta_0$, the updates $\theta(t + 1) = \theta(t) - \eta\, G_{m, \lambda}(\theta(t))$ satisfy
    \begin{align*}
        R(\theta(t), \xi) \leq (1 - \eta c)^t\, R(\theta(0), \xi), \qquad t \in \N_0,
    \end{align*}
    with $0 < 1 - \eta c < 1$. In both cases, $\theta_2(t) < 1$ for all $t$, and $\theta(t)$ converges to a parameter with zero risk.
\end{restatable}

For the initialization $\theta(0) = (\sqrt{8}, 0)$ in \Cref{prop:even-stall}, any $0 < \lambda < 2 / 3$ gives convergence for all truncation lengths. Damping replaces $\theta_2$ in the backward expansion by $1 - \lambda + \lambda\, \theta_2$. Using the conservation law, we show that this quantity remains strictly between $-1$ and $1$, so the approximate update remains a positive multiple of the exact gradient. An approximate conservation law gives the corresponding result for gradient descent. We include the proof in \Cref{app:proof-damping}.

We include experiments for the example of \Cref{prop:even-stall} in \Cref{app:approx-experiments}. With exact forward solves, two Neumann terms fail to learn the linear target, whereas JFB, three terms, and damped two terms converge to zero risk.

\subsection{Nonlinear single-index models} \label{sec:nonlinear-truncated}

Next, we show that the Neumann approximation also converges to the target for the nonlinear single-index model \eqref{eq:implicit-neuron} with target $y_\xi$.

\begin{restatable}[Convergence with fixed backward expansions, simplified]{theorem}{simfilteredsimplified} \label{thm:sim-filtered}
    Fix $m \in \N$ and suppose the assumptions of \Cref{thm:nonlinear-convergence-simplified} hold. For any initialization $\theta(0) \in B(\xi, 1 / L - \abs{\xi_2})$, there are constants $\lambda, \eta_0 > 0$ such that, for every $0 < \eta \leq \eta_0$, the updates $\theta(t + 1) = \theta(t) - \eta\, G_m(\theta(t))$ satisfy
    \begin{align*}
        \norm{\theta(t) - \xi}_2^2 \leq (1 - \eta\, \lambda / 2)^t\, \norm{\theta(0) - \xi}_2^2, \qquad t \in \N_0,
    \end{align*}
    with $0 < 1 - \eta\, \lambda / 2 < 1$. The flow $\theta^\prime = -G_m(\theta)$ exists for all $t \geq 0$ and satisfies
    \begin{align*}
        \norm{\theta(t) - \xi}_2^2 \leq e^{-\lambda t}\, \norm{\theta(0) - \xi}_2^2, \qquad t \geq 0.
    \end{align*}
    In both cases, $\abs{\theta_2(t)} < 1 / L$ for all $t$, and $R(\theta(t), \xi) \to 0$. The constants may depend on the target, initialization, activation, data distribution, and $m$.
\end{restatable}

In particular, one-term JFB and two-term backpropagation converge linearly to the target without increasing the expansion length. The proof demonstrates that for some $c > 0$, $G_m$ is aligned with the direction $\theta - \xi$ in the sense that
\begin{align*}
    \inner{G_m(\theta), \theta - \xi} \geq c\, \norm{\theta - \xi}_2^2
\end{align*}
throughout a neighborhood containing the trajectory, even though $G_m$ need not be a gradient. The general statement and proof are in \Cref{app:general-sim-filtered,app:proof-sim-filtered-simplified}. We also find that for an explicit target function ($\xi_2 = 0$), all finite backward expansions have the same linearized training dynamics as exact backpropagation; additional terms improve only the higher-order discrepancy. When $\xi_2 = 0$, the derivative of the fixed-point map with respect to its state vanishes at the target. We use this fact to show that the difference between the approximate update and the exact gradient is of order $\norm{\theta - \xi}_2^{m + 1}$, so both updates agree to first order. We give the formal statement and proof in \Cref{sec:local} of the appendix.

We include experiments comparing exact backpropagation, JFB, and fixed Neumann approximations for a nonlinear single-index target in \Cref{app:approx-experiments}. All methods learn the target, and JFB and the undamped Neumann approximations have similar loss curves to exact backpropagation. We also include CIFAR-10 experiments in \Cref{app:cifar-experiment}, showing that additional Neumann terms can worsen gradient accuracy while improving the loss decrease from a single update.

\section{Conclusion}

This work studies the gradient descent dynamics of training linear and single-index deep equilibrium models. We proved a matrix conservation law for linear DEQs, which implies global convergence from every invertible initialization. For single-index DEQs with nonlinear activation functions, we established local convergence guarantees for both gradient flow and gradient descent under suitable assumptions on the data distribution and activation function. Prior work observed that longer Neumann expansions can degrade performance due to numerical explosion of higher-order terms \citep{geng2021training}. By contrast, we prove that adding terms can cause convergence to suboptimal parameters even in a scalar linear model with \emph{exact forward solves}. We verify our theoretical results through numerical experiments in \Cref{app:experiments}. Future work could extend these results to nonlinear vector-state DEQs and inexact forward equilibria, and explore the implications for practical training of deep equilibrium models.

\printbibliography

\clearpage
\begin{appendices}
    \crefalias{section}{appendix}
    \section{Experiments} \label{app:experiments}
\begin{experimentlayout}

    In this section, we provide empirical validation of our theory. All code for the experiments can be found in the following repository:
    \begin{center}
        \texttt{\href{https://anonymous.4open.science/r/single-index-deq-7242/}{https://anonymous.4open.science/r/single-index-deq-7242/}}
    \end{center}

    \subsection{Linear models} \label{app:linear-experiments}

    We begin by considering the function $f(x) = \xi x$ with $\xi = 2$ and fitting an implicit linear model of the form
    \begin{align*}
        y_\theta(x) = g_\theta(x, y_\theta(x)) = \theta_1\, x + \theta_2\, y_\theta(x),
    \end{align*}
    running gradient descent on this implicit model starting from $\theta(0) \sim \Normal(0,\, 0.1^2\, I_2)$ using 1000 samples from $\Unif([0, 1])$. We run gradient descent for 200 epochs with a constant step size of $\eta = 0.01$. We depict the resulting loss over epochs and the learned function after training in \Cref{fig:lm-training}, showing that gradient descent is able to successfully learn the target function.

    \begin{figenv}[(a) The loss of the linear DEQ converges to zero over epochs and (b) the learned function is close to the ground truth, so the implicit model is able to learn the target function.][lm-training]
        \subfig[0.47]{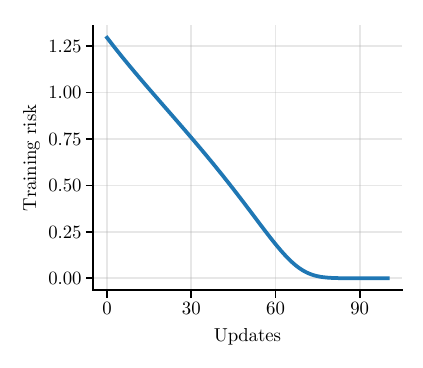}[Loss over epochs]
        \subfig[0.47]{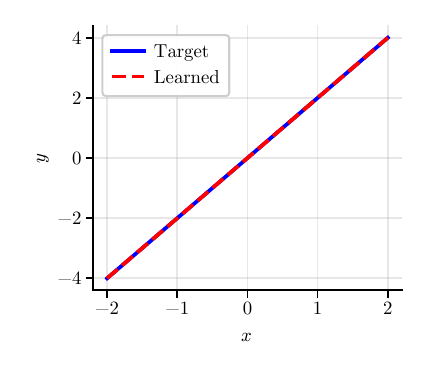}[Learned function]
    \end{figenv}

    Then, we can plot the associated trajectory of the parameters in the $(\theta_1, \theta_2)$-plane, which moves along a curved path away from the singular parameter $(0, 1)$ in \Cref{fig:lm-dynamics}; as predicted by \eqref{eq:matrix-discrete-drift}. We also plot the distance from the trivial solution over epochs, showing that the iterates move away from the trivial solution as training progresses. We observe the same phenomenon even when we initialize very close to the trivial solution by $\theta(0) \sim \Normal((0, 1),\, 0.1^2\, I_2)$, as shown in \Cref{fig:lm-unstable}.

    \begin{figenv}[(a) The gradient descent trajectory of the parameters follows a curved path around $(0, 1)$, consistent with \Cref{thm:matrix-conservation}. The line $(\alpha \xi, 1 - \alpha)$ for $\alpha \neq 0$ represents the set of valid solutions and $\theta = (0, 1)$ corresponds to the trivial model $y_\theta(x) = y_\theta(x)$. (b) In addition, the parameters move further from the trivial solution $(0, 1)$, as we demonstrated in \eqref{eq:matrix-discrete-drift}.][lm-dynamics]
        \subfig[0.49]{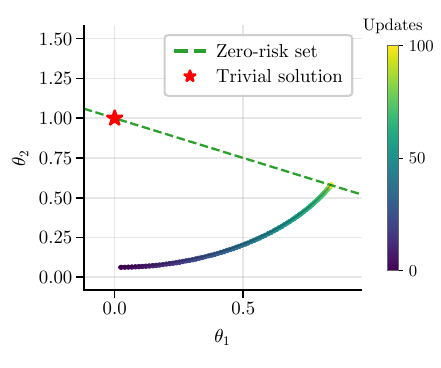}[Gradient descent trajectory]
        \subfig[0.49]{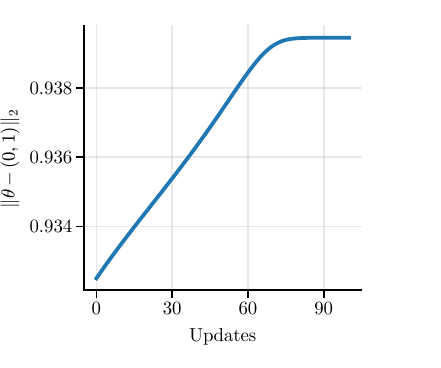}[Distance from $(0, 1)$ over epochs]
    \end{figenv}
    \begin{figenv}[Even when we initialize close to the trivial solution $(0, 1)$, we observe that the gradient descent iterates (a) converge to a zero-risk parameter and (b) move further from $(0, 1)$ over epochs. This initialization, selected for illustration, has $\theta_2 > 1$, where the equilibrium is invertible but forward fixed-point iteration is not contractive.][lm-unstable]
        \subfig[0.49]{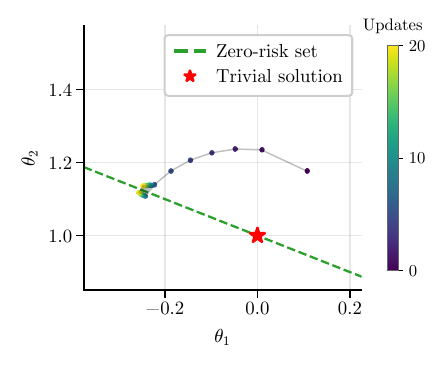}[Gradient descent trajectory]
        \subfig[0.49]{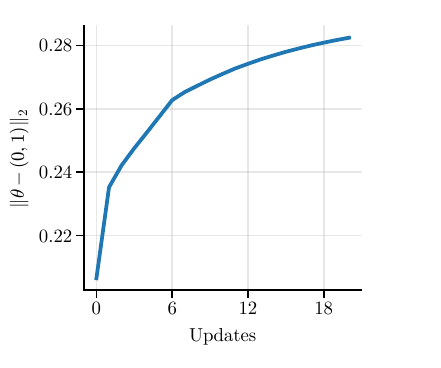}[Distance from $(0, 1)$ over epochs]
    \end{figenv}

    \subsection{Nonlinear single-index models} \label{app:nonlinear-experiments}

    Next, we consider the function $f(x) = \sigma(2x)$ with the sigmoid activation function $\sigma(z) = 1 / (1 + e^{-z})$ and fit an implicit single-index model of the form
    \begin{align*}
        y_\theta(x) = g_\theta(x, y_\theta(x)) = \sigma(\theta_1\, x + \theta_2\, y_\theta(x)),
    \end{align*}
    running gradient descent on this implicit model starting from $\theta(0) \sim \Normal(0,\, 0.1^2\, I_2)$ using 1000 training samples from $\Unif([0, 1])$. For the forward pass, we use Newton's method with a bisection fallback until the two sides of the equilibrium equation differ by at most $10^{-13}$, and check the solutions against Brent's method \citep{brent2013algorithms}. We run gradient descent for $40000$ epochs with a constant step size of $\eta = 0.1$. We depict the resulting loss over epochs and the learned function after training in \Cref{fig:deq-training}, showing that gradient descent is able to successfully learn the target function.

    \begin{figenv}[(a) The loss of the single-index DEQ converges to zero over epochs and (b) the learned function is close to the ground truth, so the implicit model is able to learn the target function.][deq-training]
        \subfig[0.47]{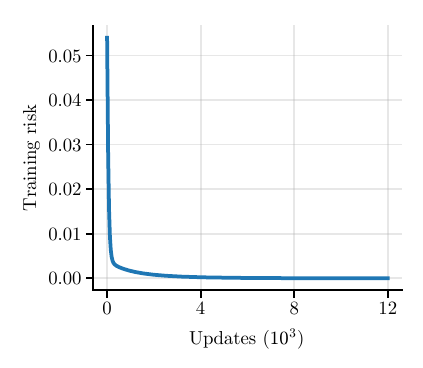}[Loss over epochs]
        \subfig[0.47]{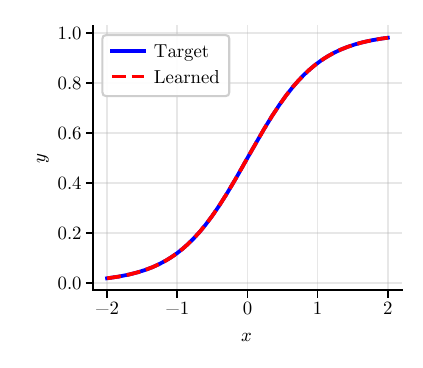}[Learned function]
    \end{figenv}

    Finally, we plot the associated trajectory of the parameters in the $(\theta_1, \theta_2)$-plane in \Cref{fig:deq-dynamics}, showing that the parameters converge to the target parameter $(2, 0)$ as training progresses. We also plot the distance from the target parameter over epochs, showing that the iterates move closer to the target parameter as training progresses. Interestingly, the parameters do not move in a straight line towards the target parameter, but rather follow a curved path. One heuristic explanation for this is that, by Taylor's theorem for small values of $\theta$,
        \begin{align*}
            \sigma(z) \approx \frac{1}{2} + \frac{z}{4}, \qquad
            y_\theta(x) \approx \frac{2 + \theta_1\, x}{4 - \theta_2}.
        \end{align*}
        Near $\theta = (0, 0)$, increasing either parameter raises the prediction towards $\sigma(2x) > 1 / 2$ for $x > 0$, so both parameters initially increase. To reach the target parameter $(2, 0)$, however, $\theta_2$ must eventually decrease back to zero, explaining why the trajectory curves back towards the $\theta_1$-axis. We leave a more rigorous analysis of this phenomenon to future work.

    \begin{figenv}[(a) The gradient descent trajectory of the parameters converges to the target parameter $(2, 0)$ and (b) the distance to the target parameter decreases over epochs.][deq-dynamics]
        \subfig[0.49]{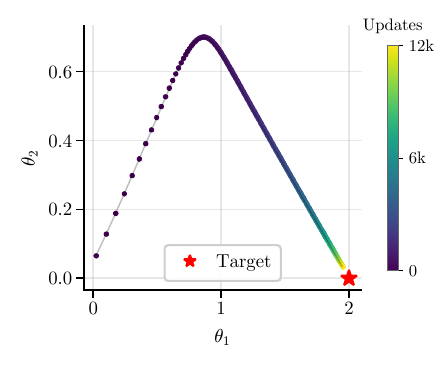}[Gradient descent trajectory]
        \subfig[0.49]{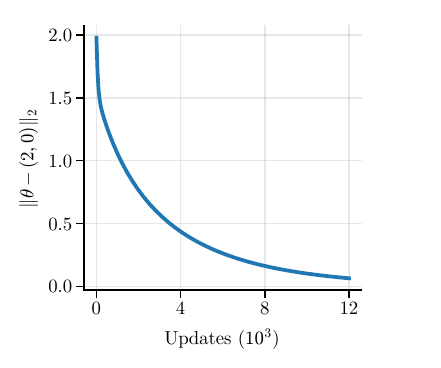}[Distance from $(2, 0)$ over epochs]
    \end{figenv}

    Finally, in \Cref{fig:deq-dynamics-far}, we initialize the parameters at $\Normal((20, 0),\, 0.1^2\, I_2)$ (far from the truth at $(2, 0)$) and use a large step size $\eta = 2$. We still observe convergence, although (informally) we only see exponentially fast convergence sufficiently close to the true parameter.

    \begin{figenv}[Gradient descent dynamics with initialization far from the truth and a large step size. (a) The loss over epochs, (b) the gradient descent trajectory of the parameters converges to the target parameter $(2, 0)$, and (c) the distance to the target parameter decreases over epochs.][deq-dynamics-far]
        \subfig[0.31]{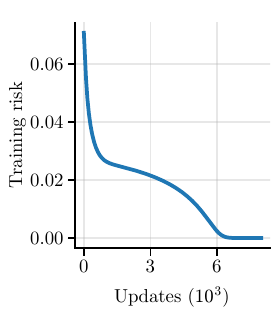}[Training risk]%
        \subfig[0.37]{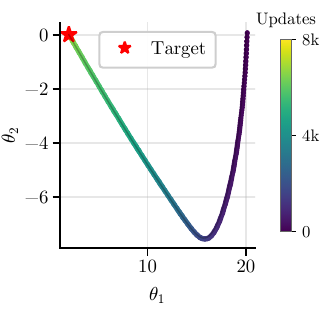}[Parameter trajectory]%
        \subfig[0.31]{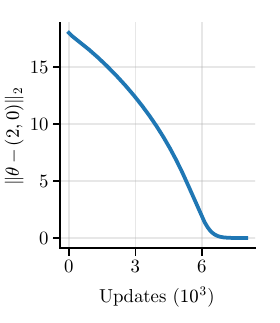}[Parameter error]%
    \end{figenv}

    \subsection{Approximate backpropagation} \label{app:approx-experiments}

    We compare exact implicit differentiation, JFB, two and three Neumann terms, and two terms with damping $\lambda = 1 / 2$, solving the forward equilibrium to the same accuracy for all methods. For the scalar example in \Cref{prop:even-stall}, we take $f(x) = x$, $X \sim \Normal(0, 1)$, and $\theta(0) = (\sqrt{8}, 0)$, computing the equilibrium and expectations analytically. We run gradient descent for $5000$ steps with $\eta = 0.02$. We depict the resulting loss in \Cref{fig:backward-training}(a), showing that two Neumann terms fail to learn the target, while the other methods converge to zero risk. Since we compute the equilibrium exactly, this failure is due to the backward approximation.

    Next, we take $d = 20$, $\sigma = \tanh$, and $\xi = (\xi_1, 0)$ with $\norm{\xi_1}_2 = 1$. We repeat the experiment five times, each with $8000$ standard Gaussian training inputs and $8000$ independent evaluation inputs, using the same data and initialization for all methods. We initialize $\theta(0) = (\xi_1 + 0.2\, v, 0.15)$ for $v$ sampled uniformly from the unit sphere, and take $40000$ full-batch steps with $\eta = 0.2$. At each step, we solve until the two sides of the equilibrium equation differ by at most $10^{-12}$. We depict the resulting loss in \Cref{fig:backward-training}(b) and report the final risk and parameter error in \Cref{tab:backward-outcomes}. All methods learn the target, with nearly identical loss curves except for the damped method, which converges more slowly at the same step size.

    \begin{figenv}[(a) Two Neumann terms fail to learn the linear target, while (b) all methods learn the nonlinear target. Nonlinear curves show means $\pm$ one sample standard deviation over five runs; risks below $10^{-14}$ are omitted.][backward-training]
        \includegraphics[width=0.98\linewidth]{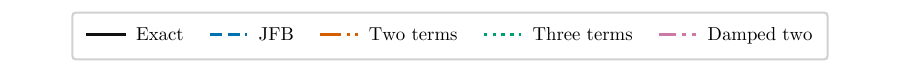}
        \par\smallskip
        \subfig[0.47]{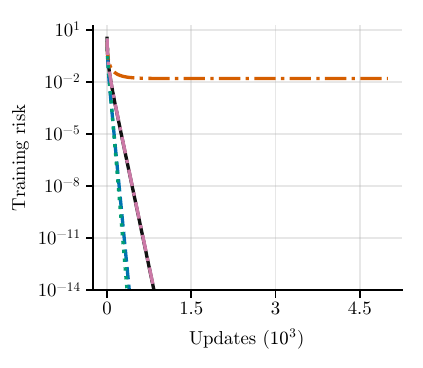}[Linear counterexample]
        \subfig[0.47]{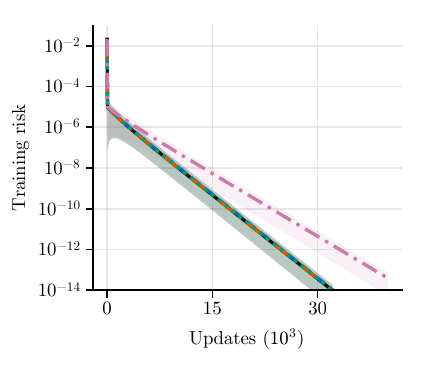}[Nonlinear single-index target]
    \end{figenv}

    \begin{table}
        \centering
        \small
        \setlength{\tabcolsep}{5pt}
        \begin{tabular}{lcc}
            \hline
            Backward rule & Held-out risk                     & Parameter error                  \\
            \hline
            Exact         & $(6.91 \pm 5.94) \times 10^{-17}$ & $(2.77 \pm 1.14) \times 10^{-7}$ \\
            JFB           & $(7.53 \pm 6.70) \times 10^{-17}$ & $(2.88 \pm 1.23) \times 10^{-7}$ \\
            Two terms     & $(6.93 \pm 5.97) \times 10^{-17}$ & $(2.78 \pm 1.14) \times 10^{-7}$ \\
            Three terms   & $(6.91 \pm 5.94) \times 10^{-17}$ & $(2.77 \pm 1.14) \times 10^{-7}$ \\
            Damped two    & $(3.99 \pm 3.42) \times 10^{-14}$ & $(6.66 \pm 2.73) \times 10^{-6}$ \\
            \hline
        \end{tabular}
        \caption{\textit{Risk and parameter error after $40000$ updates with the same step size for all methods. We report the mean $\pm$ sample standard deviation over five runs.}}
        \label{tab:backward-outcomes}
    \end{table}

    \subsection{The amount of nonlinearity} \label{app:weak-experiment}

    Next, we consider the activation $\sigma_\varepsilon(z) = z + \varepsilon \tanh z$ with $\varepsilon \in \set{1 / 2, 1 / 4, 1 / 8}$ and an explicit target $\xi = (\xi_1, 0)$ with $\norm{\xi_1}_2 = 1$ in dimension $d = 10$. We repeat the experiment five times, each with $4096$ standard Gaussian training inputs and $16384$ independent evaluation inputs. For each run, all values of $\varepsilon$ use the same data and initialization
    \begin{align*}
        \theta(0) = \xi + \frac{0.1}{\sqrt{2}}\, (-\xi_1, 1) + 0.03\, (v, 0), \qquad \norm{v}_2 = 1, \quad v \perp \xi_1.
    \end{align*}
    Here $v$ is sampled uniformly from the unit sphere in the subspace orthogonal to $\xi_1$, by projecting a standard Gaussian vector onto that subspace and normalizing it. We include the term $0.1\, (-\xi_1, 1) / \sqrt{2}$ because, when $\varepsilon = 0$, the parameters $((1 - a)\, \xi_1, a)$ represent the same function $\xi_1^\top x$ for every $a \neq 1$. This initialization lets us examine how nonlinearity helps gradient descent recover the target parameter, rather than just the target function. We use exact implicit differentiation and run gradient descent for $60000$ steps with $\eta = 0.05$, using all training samples at each step.

    We depict the parameter error and risk in \Cref{fig:weak-nonlinearity}, showing that weaker nonlinearity slows parameter recovery even when prediction risk is small. We also compute the Hessian of the empirical risk at the target parameter. Averaged over the five runs, its smallest eigenvalue is approximately $0.0276\, \varepsilon^2$ for all three values of $\varepsilon$, consistent with \Cref{rem:nonlinearity}.

    \begin{figenv}[(a) Parameter recovery slows as $\varepsilon$ decreases, although (b) prediction risk is small. Curves show means and shading shows the range over five runs, using the same data, initialization, and step size for each value of $\varepsilon$.][weak-nonlinearity]
        \subfig[0.47]{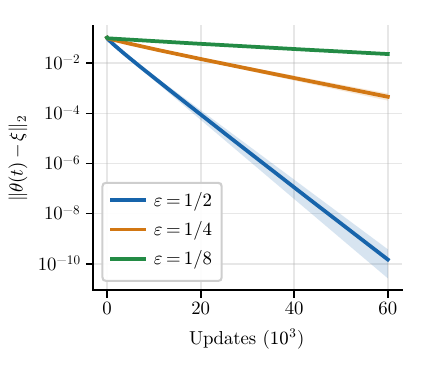}[Parameter error]
        \subfig[0.47]{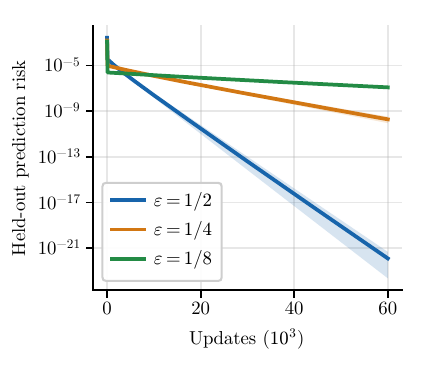}[Held-out prediction risk]
    \end{figenv}

    \subsection{Multiscale DEQs with finite forward solves} \label{app:cifar-experiment}

    Finally, we examine how the number of Neumann terms affects training and gradient accuracy on CIFAR-10. We train a small multiscale DEQ, using two $24$-channel branches, weight normalization, and the same random seed for all methods. We use $45000$ images for training, $5000$ for validation, and $10000$ for testing. All methods use $50$ epochs, batch size $128$, no augmentation, and untuned SGD with momentum $0.9$, weight decay $10^{-4}$, and initial learning rate $0.1$ under a common cosine schedule. For comparison, we also include \emph{phantom gradients} \citep{geng2021training}: starting from the computed equilibrium, we take a few additional damped fixed-point iterations and backpropagate only through these iterations, treating their initial state as constant. We use five iterations with damping $\lambda = 1 / 2$.

    For the forward pass, we use at most $22$ Broyden iterations with absolute tolerance $10^{-3}$ and do not reject inaccurate solves. We depict the training loss in \Cref{fig:cifar-budget} and report the final accuracy in \Cref{tab:cifar-budget}. Three terms and damped two terms give higher training loss, lower test accuracy, and less accurate forward solves. To measure how closely the computed state $z$ satisfies the equilibrium equation, we use
    \begin{align*}
        \frac{\norm{F_\theta(z, x) - z}_2}{\norm{F_\theta(z, x)}_2 + 10^{-9}},
    \end{align*}
    where $F_\theta$ is the multiscale DEQ map. We average this quantity over each batch and report the largest batch average per epoch, before the additional phantom-gradient iterations.

    \begin{figenv}[(a) Training loss and (b) forward residuals on CIFAR-10 with at most $22$ forward iterations. With this fixed forward-solver budget, three terms and damped two terms give higher training loss than JFB.][cifar-budget]
        \includegraphics[width=0.98\linewidth]{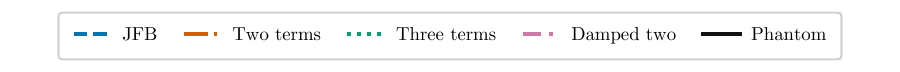}
        \par\smallskip
        \subfig[0.47]{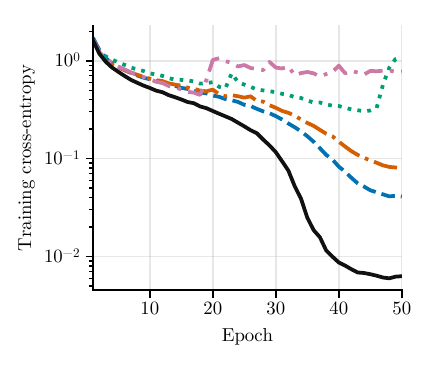}[Training loss]
        \subfig[0.47]{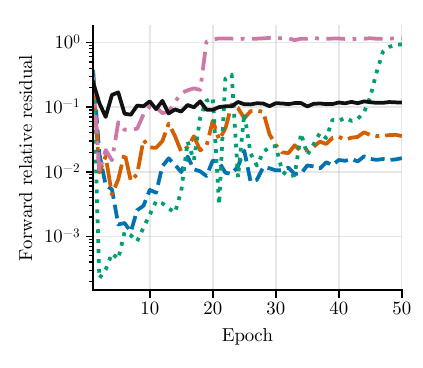}[Forward residual]
    \end{figenv}

    \begin{table}
        \centering
        \small
        \begin{tabular}{lccc}
            \hline
            Backward rule & Validation accuracy & Test accuracy & Forward residual \\
            \hline
            JFB           & $82.22\%$           & $81.61\%$     & $0.0162$         \\
            Two terms     & $83.48\%$           & $83.01\%$     & $0.0356$         \\
            Three terms   & $60.28\%$           & $60.00\%$     & $0.9324$         \\
            Damped two    & $71.80\%$           & $71.94\%$     & $1.1470$         \\
            Phantom       & $86.60\%$           & $85.33\%$     & $0.1179$         \\
            \hline
        \end{tabular}
        \caption{\textit{Final accuracy and forward residuals from one training run per method, using at most $22$ forward iterations.}}
        \label{tab:cifar-budget}
    \end{table}

    \paragraph{Accurate forward solves.} In the preceding training runs, both the backward approximation and the accuracy of the forward solves vary across methods. To isolate the effect of the backward approximation, we next compare a single update from the same parameters and accurately computed equilibrium. We use the final parameters from the JFB and two-term runs, each on three batches of $128$ images. We retain five of these six comparisons; in the remaining case, the forward solver does not reach the required accuracy. We compute a reference gradient by implicit differentiation, checking that the computed state and derivative satisfy their defining equations to relative tolerances $10^{-7}$ and $10^{-10}$, respectively. We report the gradient error and loss decrease in \Cref{tab:cifar-local}, showing that additional Neumann terms can worsen gradient accuracy while improving the loss decrease from a single update.

    \begin{table}
        \centering
        \small
        \begin{tabular}{lcc}
            \hline
            Quantity                          & Two terms        & Three terms      \\
            \hline
            Relative gradient error           & $[0.74,\, 0.95]$ & $[0.91,\, 1.49]$ \\
            Loss decrease / two-term decrease & $1$              & $[1.17,\, 1.25]$ \\
            \hline
        \end{tabular}
        \caption{\textit{Ranges over the five comparisons with accurate forward solves. We measure gradient error for the parameters of the equilibrium layers, relative to the gradient computed by implicit differentiation. Both methods use the same step size in each comparison, chosen so that the relative change in parameters is at most $10^{-5}$.}}
        \label{tab:cifar-local}
    \end{table}

\end{experimentlayout}

    \section{Proofs for multivariate linear DEQs} \label{app:multivariate-deqs}

All parameter norms in this appendix are product Frobenius norms. We write $\lambda_\mathrm{max}(A)$ for the maximum eigenvalue of a symmetric matrix $A$.

\subsection[Proof of matrix conservation]{Proof of \Cref{thm:matrix-conservation}} \label{app:proof-matrix-conservation}

\matrixconservation*

\begin{proof}
    Differentiating $y_\Gamma(x) = (I_m - \Gamma_2)^{-1}\, \Gamma_1 x$ with respect to each parameter gives
    \begin{align*}
        D_{\Gamma_1}y_\Gamma(x)[U] & = (I_m - \Gamma_2)^{-1}\, Ux,           \\
        D_{\Gamma_2}y_\Gamma(x)[V] & = (I_m - \Gamma_2)^{-1}\, Vy_\Gamma(x).
    \end{align*}
    Since $X \in L^2(\P)$, dominated convergence justifies differentiating the risk. Thus, writing
    \begin{align*}
        G \coloneq \grad_{\Gamma_1}\, R
        = 2\, (I_m - \Gamma_2)^{-\top}\, ((I_m - \Gamma_2)^{-1}\, \Gamma_1 - \Xi)\, \Sigma,
    \end{align*}
    we have $\grad_{\Gamma_2}\, R = G\, \Gamma_1^\top\, (I_m - \Gamma_2)^{-\top}$. The gradient-flow equations are therefore
    \begin{equation} \label{eq:matrix-gradient-flow}
        \Gamma_1^\prime = -G, \qquad
        \Gamma_2^\prime = -G\, \Gamma_1^\top\, (I_m - \Gamma_2)^{-\top}.
    \end{equation}
    Applying the product rule gives
    \begin{align*}
        \frac{d}{dt}\bigl((I_m - \Gamma_2)(I_m - \Gamma_2)^\top\bigr)
         & = -\Gamma_2^\prime\, (I_m - \Gamma_2)^\top - (I_m - \Gamma_2)\, (\Gamma_2^\prime)^\top \\
         & = G\, \Gamma_1^\top + \Gamma_1\, G^\top,
    \end{align*}
    whereas
    \begin{align*}
        \frac{d}{dt}(\Gamma_1\Gamma_1^\top)
        = -G\, \Gamma_1^\top - \Gamma_1\, G^\top.
    \end{align*}
    Adding the two identities proves \eqref{eq:matrix-conservation}.
\end{proof}

\subsection[Proof of multivariate gradient-flow convergence]{Proof of \Cref{thm:matrix-flow}} \label{app:proof-matrix-flow}

\matrixflow*

\begin{proof}
    If $\Sigma = 0$, the risk and all gradients vanish, so the parameters are stationary. We assume $\Sigma \neq 0$ and let $Q$ be the orthogonal projection onto $\operatorname{range}(\Sigma)$. Since $X = QX$ almost surely,
    \begin{align*}
        R(\Gamma, \Xi)
        = \tr\bigl(((I_m - \Gamma_2)^{-1}\, \Gamma_1 - \Xi)\, \Sigma\,
        ((I_m - \Gamma_2)^{-1}\, \Gamma_1 - \Xi)^\top\bigr).
    \end{align*}
    We write $\lambda_\mathrm{min}^{+}(\Sigma)$ for the smallest strictly positive eigenvalue of $\Sigma$, and define
    \begin{align*}
        H(t) & \coloneq (I_m - \Gamma_2(t))(I_m - \Gamma_2(t))^\top + \Gamma_1(t)\, Q\, \Gamma_1(t)^\top, \\
        R_0  & \coloneq R(\Gamma(0), \Xi), \qquad
        K \coloneq \norm{\Xi Q}_\mathrm{op} + \sqrt{\frac{R_0}{\lambda_\mathrm{min}^{+}(\Sigma)}}.
    \end{align*}
    The matrix $H(0)$ is positive definite because $I_m - \Gamma_2(0)$ is invertible. In \eqref{eq:matrix-gradient-flow}, $G = GQ$, so $\Gamma_1(I_d - Q)$ is constant. Subtracting this constant contribution from \eqref{eq:matrix-conservation} gives $H(t) = H(0)$.

    For as long as the solution exists, the risk is nonincreasing and
    \begin{align*}
        \lambda_\mathrm{min}^{+}(\Sigma)\, \norm{((I_m - \Gamma_2)^{-1}\, \Gamma_1 - \Xi)\, Q}_F^2
                                                              & \leq R(\Gamma, \Xi) \leq R_0, \\
        \norm{(I_m - \Gamma_2)^{-1}\, \Gamma_1 Q}_\mathrm{op} & \leq K.
    \end{align*}
    The second inequality implies $\Gamma_1 Q\Gamma_1^\top \preceq K^2(I_m - \Gamma_2)(I_m - \Gamma_2)^\top$. Therefore
    \begin{align*}
        \frac{\lambda_\mathrm{min}(H(0))}{1 + K^2}\, I_m
        \preceq (I_m - \Gamma_2)(I_m - \Gamma_2)^\top \preceq H(0).
    \end{align*}
    These bounds keep $I_m - \Gamma_2$ away from singularity and bound both $I_m - \Gamma_2$ and $\Gamma_1Q$. Since $\Gamma_1(I_d - Q)$ is constant, the trajectory remains in a compact subset of the domain of the smooth vector field. The flow thus exists for all $t \geq 0$.

    Differentiating the risk along the flow gives
    \begin{align*}
        R^\prime
         & = -\norm{G}_F^2 - \norm{G\Gamma_1^\top(I_m - \Gamma_2)^{-\top}}_F^2                \\
         & \leq -\frac{4}{\lambda_\mathrm{max}(H(0))}\,
        \norm{((I_m - \Gamma_2)^{-1}\, \Gamma_1 - \Xi)\, \Sigma}_F^2                          \\
         & \leq -\frac{4\, \lambda_\mathrm{min}^{+}(\Sigma)}{\lambda_\mathrm{max}(H(0))}\, R.
    \end{align*}
    The last inequality follows from $\Sigma^2 \succeq \lambda_\mathrm{min}^{+}(\Sigma)\, \Sigma$. Multiplying by
    $\exp(4 \lambda_\mathrm{min}^{+}(\Sigma)\, t / \lambda_\mathrm{max}(H(0)))$
    and integrating gives
    \begin{align*}
        R(\Gamma(t), \Xi) \leq R_0\,
        \exp\left(-\frac{4 \lambda_\mathrm{min}^{+}(\Sigma)}{\lambda_\mathrm{max}(H(0))}\, t\right).
    \end{align*}
    To prove convergence of the parameters, we first note that
    \begin{align*}
        \norm{((I_m - \Gamma_2)^{-1}\, \Gamma_1 - \Xi)\, \Sigma}_F
         & \leq \sqrt{\lambda_\mathrm{max}(\Sigma)\, R},                                                        \\
        \norm{\Gamma^\prime}_F
         & \leq 2(1 + K^2)\, \sqrt{\frac{\lambda_\mathrm{max}(\Sigma)}{\lambda_\mathrm{min}(H(0))}}\, \sqrt{R}.
    \end{align*}
    We write $c = 4 \lambda_\mathrm{min}^{+}(\Sigma) / \lambda_\mathrm{max}(H(0))$ and
    $C = 2(1 + K^2)\sqrt{\lambda_\mathrm{max}(\Sigma) / \lambda_\mathrm{min}(H(0))}$.
    For $s \geq t$ we then have
    \begin{align*}
        \norm{\Gamma(s) - \Gamma(t)}_F
        \leq C \sqrt{R_0} \int_t^s e^{-cu / 2}\, du
        \leq \frac{2C \sqrt{R_0}}{c}\, e^{-ct / 2}.
    \end{align*}
    Thus $\Gamma(t)$ has a limit $\Gamma^\infty$. The lower bound on
    $(I_m - \Gamma_2(t))(I_m - \Gamma_2(t))^\top$ implies that
    $I_m - \Gamma_2^\infty$ is invertible. The risk is continuous at this limit, so
    $R(\Gamma^\infty, \Xi) = 0$. If $\Sigma \succ 0$, the risk formula then gives
    $(I_m - \Gamma_2^\infty)^{-1}\, \Gamma_1^\infty = \Xi$.
\end{proof}
\subsection[Proof of multivariate gradient-descent convergence]{Proof of \Cref{thm:matrix-descent}} \label{app:proof-matrix-descent}

\matrixdescent*

\begin{proof}
    If $\Sigma = 0$, the risk and all gradients vanish. Otherwise, we use $Q$, $H$, $R_0$, and $K$ from \Cref{app:proof-matrix-flow}. It suffices to take
    \begin{align*}
        0 < \eta \leq \frac{\lambda_\mathrm{min}(H(0))}{96 \lambda_\mathrm{max}(\Sigma)(1 + K^2)^2}.
    \end{align*}
    For this choice, we prove
    \begin{align*}
        R(\Gamma(t), \Xi)
        \leq \left(1 - \frac{2 \eta\, \lambda_\mathrm{min}^{+}(\Sigma)}
        {\lambda_\mathrm{max}(H(0)) + 2 \eta R_0} \right)^t\, R_0.
    \end{align*}
    If $R_0 = 0$, all iterates are stationary, so we may assume $R_0 > 0$. At any well-defined iterate, we write
    \begin{align*}
        G(t) = 2(I_m - \Gamma_2(t))^{-\top}((I_m - \Gamma_2(t))^{-1} \Gamma_1(t) - \Xi)\, \Sigma.
    \end{align*}
    The updates are
    \begin{align*}
        \Gamma_1(t + 1) & = \Gamma_1(t) - \eta G(t),                                              \\
        \Gamma_2(t + 1) & = \Gamma_2(t) - \eta G(t) \Gamma_1(t)^\top (I_m - \Gamma_2(t))^{-\top}.
    \end{align*}
    Since $G(t) = G(t)\, Q$, the component $\Gamma_1(t)(I_d - Q)$ is constant. The first-order terms in $H(t + 1) - H(t)$ cancel exactly, giving
    \begin{equation} \label{eq:matrix-discrete-drift}
        \begin{split}
            H(t + 1) - H(t)
                                 & = \eta^2 G(t)\, G(t)^\top                                                                                \\
                                 & \quad + \eta^2 \grad_{\Gamma_2}\, R(\Gamma(t), \Xi) \grad_{\Gamma_2}\, R(\Gamma(t), \Xi)^\top \succeq 0, \\
            \tr(H(t + 1) - H(t)) & = \eta^2\, \norm{\grad R(\Gamma(t), \Xi)}_F^2.
        \end{split}
    \end{equation}

    We next establish a descent estimate uniformly over every iterate satisfying $R(\Gamma(t), \Xi) \leq R_0$ and $H(t) \succeq H(0)$. The risk and matrix inequalities in \Cref{app:proof-matrix-flow} give
    \begin{align*}
        \norm{(I_m - \Gamma_2(t))^{-1} \Gamma_1(t)\, Q}_\mathrm{op} \leq K, \qquad \norm{(I_m - \Gamma_2(t))^{-1}}_\mathrm{op}^2 \leq \frac{1 + K^2}{\lambda_\mathrm{min}(H(0))}.
    \end{align*}
    We omit the time argument and parameterize the update segment by
    \begin{align*}
        \Gamma_1(u) & = \Gamma_1 - u \eta G,                                                                \\
        \Gamma_2(u) & = \Gamma_2 - u \eta G \Gamma_1^\top (I_m - \Gamma_2)^{-\top}, \qquad 0 \leq u \leq 1.
    \end{align*}
    We have
    \begin{align*}
        \norm{G}_F \leq 2 \sqrt{\frac{\lambda_\mathrm{max}(\Sigma)\, R_0(1 + K^2)}{\lambda_\mathrm{min}(H(0))}}, \qquad \sqrt{\frac{R_0}{\lambda_\mathrm{max}(\Sigma)}} \leq K.
    \end{align*}
    The prescribed step size is, in particular, at most $\lambda_\mathrm{min}(H(0)) / [4 \lambda_\mathrm{max}(\Sigma)(1 + K^2)^2]$. It follows that
    \begin{align*}
        \norm{\Gamma_2(u) - \Gamma_2}_\mathrm{op}
         & \leq 2 \eta K
        \sqrt{\frac{\lambda_\mathrm{max}(\Sigma)\, R_0(1 + K^2)}{\lambda_\mathrm{min}(H(0))}} \\
         & \leq \frac{1}{2} \sqrt{\frac{\lambda_\mathrm{min}(H(0))}{1 + K^2}}.
    \end{align*}
    Thus the entire segment is invertible and
    \begin{align*}
        \norm{(I_m - \Gamma_2(u))^{-1}}_\mathrm{op} \leq 2 \sqrt{\frac{1 + K^2}{\lambda_\mathrm{min}(H(0))}}.
    \end{align*}
    Furthermore,
    \begin{align*}
         & (I_m - \Gamma_2(u))^{-1} \Gamma_1(u)\, Q - (I_m - \Gamma_2)^{-1} \Gamma_1\, Q \\
         & \qquad = -u \eta (I_m - \Gamma_2(u))^{-1} G
        \bigl(I_d + Q \Gamma_1^\top (I_m - \Gamma_2)^{-\top}(I_m - \Gamma_2)^{-1} \Gamma_1\, Q \bigr),
    \end{align*}
    so
    \begin{align*}
         & \norm{(I_m - \Gamma_2(u))^{-1} \Gamma_1(u)\, Q - (I_m - \Gamma_2)^{-1} \Gamma_1\, Q}_F                     \\
         & \qquad \leq \frac{4 \eta \sqrt{\lambda_\mathrm{max}(\Sigma)\, R_0}(1 + K^2)^2}{\lambda_\mathrm{min}(H(0))}
        \leq \sqrt{\frac{R_0}{\lambda_\mathrm{max}(\Sigma)}} \leq K.
    \end{align*}
    Consequently, $\norm{(I_m - \Gamma_2(u))^{-1} \Gamma_1(u)\, Q}_\mathrm{op} \leq 2K$ and
    \begin{align*}
        \sqrt{R(\Gamma(u), \Xi)}
         & \leq \sqrt{R_0} + \sqrt{\lambda_\mathrm{max}(\Sigma)}
        \norm{(I_m - \Gamma_2(u))^{-1} \Gamma_1(u)\, Q - (I_m - \Gamma_2)^{-1} \Gamma_1\, Q}_F \\
         & \leq 2 \sqrt{R_0}.
    \end{align*}

    We next bound the Hessian along this segment. A perturbation $W = (U, V)$ in $(\Gamma_1, \Gamma_2)$ satisfies
    \begin{align*}
        D((I_m - \Gamma_2)^{-1} \Gamma_1\, Q)[W]
         & = (I_m - \Gamma_2)^{-1}(UQ + V(I_m - \Gamma_2)^{-1} \Gamma_1\, Q),     \\
        D^2((I_m - \Gamma_2)^{-1} \Gamma_1\, Q)[W, W]
         & = 2(I_m - \Gamma_2)^{-1} V\, D((I_m - \Gamma_2)^{-1} \Gamma_1\, Q)[W].
    \end{align*}
    Consequently, for $\norm{W}_F = 1$,
    \begin{align*}
        \abs{D^2 R[W, W]}
         & \leq 2 \lambda_\mathrm{max}(\Sigma)\, \norm{(I_m - \Gamma_2)^{-1}}_\mathrm{op}^2
        (1 + \norm{(I_m - \Gamma_2)^{-1} \Gamma_1 Q}_\mathrm{op}^2)                         \\
         & \quad + 4 \norm{(I_m - \Gamma_2)^{-1}}_\mathrm{op}^2
        \sqrt{\lambda_\mathrm{max}(\Sigma)\, R}
        \sqrt{1 + \norm{(I_m - \Gamma_2)^{-1} \Gamma_1 Q}_\mathrm{op}^2}.
    \end{align*}
    Using the bounds just proved along the segment gives
    \begin{align*}
        \norm{\grad^2 R}_\mathrm{op}
         & \leq \frac{8 \lambda_\mathrm{max}(\Sigma)(1 + K^2)}{\lambda_\mathrm{min}(H(0))}
        \left(1 + 4K^2 + 4K \sqrt{1 + 4K^2} \right)                                            \\
         & \leq \frac{96 \lambda_\mathrm{max}(\Sigma)(1 + K^2)^2}{\lambda_\mathrm{min}(H(0))}.
    \end{align*}
    The second inequality uses $4K \sqrt{1 + 4K^2} \leq 1 + 8K^2$. Our step size is at most the reciprocal of this Hessian bound. Taylor's theorem along the segment therefore gives
    \begin{equation} \label{eq:matrix-descent-bound}
        R(\Gamma(t + 1), \Xi) \leq R(\Gamma(t), \Xi) - \frac{\eta}{2} \norm{\grad R(\Gamma(t), \Xi)}_F^2.
    \end{equation}
    In particular, $I_m - \Gamma_2(t + 1)$ is invertible, $R(\Gamma(t + 1), \Xi) \leq R_0$, and $H(t + 1) \succeq H(0)$. Induction proves these properties for every iterate. At every iterate, we can now use $R(\Gamma(t), \Xi) \leq R_0$ and $H(t) \succeq H(0)$ to obtain
    \begin{align*}
        \sigma_\mathrm{min}((I_m - \Gamma_2(t)))^2 \geq \frac{\lambda_\mathrm{min}(H(0))}{1 + K^2}.
    \end{align*}

    Summing \eqref{eq:matrix-discrete-drift} and \eqref{eq:matrix-descent-bound} gives
    \begin{align*}
        \tr(H(t) - H(0)) = \eta^2 \sum_{j < t}\, \norm{\grad R(\Gamma(j), \Xi)}_F^2 \leq 2 \eta \bigl(R_0 - R(\Gamma(t), \Xi) \bigr) \leq 2 \eta R_0.
    \end{align*}
    Thus $\norm{(I_m - \Gamma_2(t))}_\mathrm{op}^2 \leq \lambda_\mathrm{max}(H(0)) + 2 \eta R_0$ and
    \begin{align*}
        \norm{\grad R(\Gamma(t), \Xi)}_F^2
         & \geq \norm{G(t)}_F^2
        \geq \frac{4 \norm{((I_m - \Gamma_2(t))^{-1} \Gamma_1(t) - \Xi)\, \Sigma}_F^2}{\lambda_\mathrm{max}(H(0)) + 2 \eta R_0} \\
         & \geq \frac{4 \lambda_\mathrm{min}^{+}(\Sigma)}{\lambda_\mathrm{max}(H(0)) + 2 \eta R_0}\, R(\Gamma(t), \Xi).
    \end{align*}
    Combining this with \eqref{eq:matrix-descent-bound} and applying induction gives the stated risk bound. The factor in this bound lies in $(0, 1)$ because
    \begin{align*}
        \frac{2 \eta\, \lambda_\mathrm{min}^{+}(\Sigma)}{\lambda_\mathrm{max}(H(0)) + 2 \eta R_0} \leq \frac{2 \eta\, \lambda_\mathrm{max}(\Sigma)}{\lambda_\mathrm{min}(H(0))} \leq \frac{1}{48}.
    \end{align*}
    Finally,
    \begin{align*}
        \norm{\Gamma(t + 1) - \Gamma(t)}_F \leq 2 \eta(1 + K^2) \sqrt{\frac{\lambda_\mathrm{max}(\Sigma)}{\lambda_\mathrm{min}(H(0))}} \sqrt{R(\Gamma(t), \Xi)}.
    \end{align*}
    The risk bound implies that the sum of these increments is bounded by a convergent geometric series. More explicitly, writing $a \in (0, 1)$ for the factor in the risk bound and $C$ for the constant multiplying $\sqrt{R(\Gamma(t), \Xi)}$ above, for $s > t$ we have
    \begin{align*}
        \norm{\Gamma(s) - \Gamma(t)}_F
        \leq C \sqrt{R_0} \sum_{j = t}^{s - 1} a^{j / 2}
        \leq \frac{C \sqrt{R_0}}{1 - \sqrt{a}}\, a^{t / 2}.
    \end{align*}
    Hence $\Gamma(t)$ converges to a limit $\Gamma^\infty$. The uniform lower bound on
    $\sigma_\mathrm{min}(I_m - \Gamma_2(t))$ shows that $I_m - \Gamma_2^\infty$ is invertible.
    Continuity of the risk gives $R(\Gamma^\infty, \Xi) = 0$, which, if $\Sigma \succ 0$, implies
    $(I_m - \Gamma_2^\infty)^{-1}\, \Gamma_1^\infty = \Xi$.
\end{proof}
\subsection[Arbitrary square-integrable labels]{Proof of \Cref{cor:matrix-general-labels}} \label{app:proof-matrix-general-labels}

\matrixgenerallabels*

\begin{proof}
    Let $Q$ be the orthogonal projection onto $\operatorname{range}(\Sigma)$. Since $X = QX$ almost surely,
    \begin{align*}
        \E[YX^\top](I_d - Q) = 0, \qquad \Xi Q = \Xi, \qquad \Xi \Sigma = \E[YX^\top].
    \end{align*}
    Expanding the square in $\mc{L}(\Gamma)$, the cross term vanishes because
    \begin{align*}
         & \E\left[\bigl((I_m - \Gamma_2)^{-1}\, \Gamma_1 X - \Xi X\bigr)^\top(\Xi X - Y)\right] \\
         & \qquad = \tr\left(\bigl((I_m - \Gamma_2)^{-1}\, \Gamma_1 - \Xi\bigr)
        \bigl(\Sigma \Xi^\top - \E[XY^\top]\bigr)\right) = 0.
    \end{align*}
    Therefore
    \begin{align*}
        \mc{L}(\Gamma)
         & = \E[\norm{\Xi X - Y}_2^2]                                                                                           \\
         & \quad + \tr \bigl(((I_m - \Gamma_2)^{-1} \Gamma_1 - \Xi)\, \Sigma((I_m - \Gamma_2)^{-1} \Gamma_1 - \Xi)^\top \bigr).
    \end{align*}
    The first term is constant in the parameters and is the smallest risk among linear predictors. We can therefore apply \Cref{thm:matrix-flow,thm:matrix-descent} to the excess risk. For a finite empirical distribution, the required second moments are automatically finite.
\end{proof}

\section{Proofs for single-index DEQs and backward approximations} \label{app:proofs}

\subsection{General gradient-flow convergence} \label{app:general-nonlinear-convergence}

\begin{restatable}[Gradient flow converges for single-index DEQs]{theorem}{nonlinearflowgeneral} \label{thm:nonlinear-convergence}
    Fix constants $\delta_1 > 0$ and $\delta_2 \in (\abs{\xi_2}, 1 / L)$ and assume that
    \begin{enumerate}
        \item The data satisfies $X \in L^2(\P)$ and, for some $\alpha > 0$, the target features satisfy
              \begin{equation} \label{eq:target-Gram}
                  \Sigma_\xi \coloneq \E \left[
                  \begin{pmatrix} X \\ y_\xi(X) \end{pmatrix}
                  \begin{pmatrix} X \\ y_\xi(X) \end{pmatrix}^{\!\top}
                  \indicator_{\norm{X}_2 \leq \alpha} \right] \succ 0.
              \end{equation}
        \item The activation $\sigma \in C^1(\R)$ is nondecreasing with $\norm{\sigma}_\mathrm{Lip} \leq L \in (0, \infty)$ and satisfies
              \begin{align*}
                  \inf_{\abs{z} \leq \alpha(\norm{\xi_1}_2 + \delta_1) + M \delta_2} \sigma^\prime(z) \geq \gamma > 0, \qquad M \coloneq \frac{\abs{\sigma(0)} + L \alpha(\norm{\xi_1}_2 + \delta_1)}{1 - L \delta_2}.
              \end{align*}
        \item The target is realizable and $\abs{\xi_2} < 1 / L$.
    \end{enumerate}
    Define
    \begin{align*}
        \Theta \coloneq \set{\theta : \norm{\theta_1}_2 \leq \norm{\xi_1}_2 + \delta_1, \ \abs{\theta_2} \leq \delta_2}, \qquad \rho \coloneq \lambda_\mathrm{min}(\Sigma_\xi) > 0,
    \end{align*}
    and put
    \begin{align*}
        \lambda_1 \coloneq \frac{4 \rho \gamma^2(1 - \abs{\xi_2}\, L)}{(1 + L \delta_2)^3}.
    \end{align*}
    If $\norm{\theta(0) - \xi}_2 \leq r \coloneq \min \set{\delta_1, \delta_2 - \abs{\xi_2}}$, gradient flow has a solution defined for all $t \geq 0$, and every solution extends to all $t \geq 0$ and satisfies
    \begin{align*}
        \norm{\theta(t) - \xi}_2^2 \leq \norm{\theta(0) - \xi}_2^2\, e^{-\lambda_1 t}, \qquad t \geq 0.
    \end{align*}
    In particular, the parameters remain in $\overline B(\xi, r) \subseteq \Theta$ and converge to $\xi$. The solution is unique if the population gradient is locally Lipschitz, for example if $\sigma \in C^2$ and $X$ is bounded.
\end{restatable}

The proof in \Cref{app:proof-nonlinear-convergence} bounds the component of the gradient toward the target. The lower bound on the rate is proportional to $\lambda_\mathrm{min}(\Sigma_\xi)$ and tends to zero as $\abs{\xi_2} \uparrow 1 / L$.

\begin{remark}[Discussion of assumptions]
    The condition $\Sigma_\xi \succ 0$ excludes linear dependence between $y_\xi(X)$ and the coordinates of $X$ when $\norm{X}_2 \leq \alpha$. It is a condition on the target rather than the initialization; for example, it allows $\sigma = \tanh$ and $\theta_1 = 0$ at initialization when $\Sigma_\xi \succ 0$.

    The derivative assumption excludes the step function and ReLU, but holds for sigmoid, hyperbolic tangent, and softplus on compact intervals. The assumptions in \Cref{thm:nonlinear-convergence-simplified} give sufficient conditions on the activation and data distribution for $\Sigma_\xi \succ 0$.

\end{remark}

\subsection{General gradient-descent convergence} \label{app:general-nonlinear-gd}

\begin{restatable}[Gradient descent for single-index DEQs]{theorem}{nonlineardescentgeneral} \label{thm:nonlinear-gd-convergence-general}
    Under the assumptions and initialization of \Cref{thm:nonlinear-convergence}, define
    \begin{align*}
        S \coloneq \sup_{\theta \in \Theta} \E \left[\norm*{\begin{pmatrix} X \\ y_\theta(X) \end{pmatrix}}_2^2 \right],
        \qquad
        \lambda_2 \coloneq \frac{4 L^4 S^2}{(1 - L \delta_2)^2(1 - L \abs{\xi_2})^2}.
    \end{align*}
    These constants are finite and positive. For $0 < \eta \leq \lambda_1 / (2 \lambda_2)$, all gradient descent iterates remain in $\overline B(\xi, r)$ and
    \begin{align*}
        \norm{\theta(t) - \xi}_2^2 \leq \left(1 - \frac{\eta\, \lambda_1}{2} \right)^t\, \norm{\theta(0) - \xi}_2^2, \qquad t \in \N_0.
    \end{align*}
\end{restatable}

We give the proof in \Cref{app:proof-nonlinear-gd-general}. The risk bound in \Cref{lem:single-index-bounds} also shows that convergence of the parameters implies convergence of the risk.

\subsection{General convergence with fixed backward expansions} \label{app:general-sim-filtered}

\begin{restatable}[Convergence with fixed backward expansions]{theorem}{simfilteredgeneral} \label{thm:sim-filtered-general}
    Suppose the assumptions of \Cref{thm:nonlinear-convergence} hold. Fix $m \in \N$, put $q = L \delta_2 < 1$, and define
    \begin{align*}
        p_- = \min_{\abs{a} \leq q} p_m(a) > 0, \qquad p_+ = \max_{\abs{a} \leq q} p_m(a),
    \end{align*}
    \begin{align*}
        \lambda_1 = \frac{4 \rho \gamma^2(1 - L \abs{\xi_2})\, p_-}{(1 + q)^2}, \qquad \lambda_2 = \frac{4 L^4 p_+^2 S^2}{(1 - L \abs{\xi_2})^2},
    \end{align*}
    where $S$ is defined in \Cref{thm:nonlinear-gd-convergence-general}. For every $\theta \in \Theta$,
    \begin{equation} \label{eq:target-directed-bounds}
        \inner{G_m(\theta), \theta - \xi} \geq \frac{\lambda_1}{2} \norm{\theta - \xi}_2^2, \qquad \norm{G_m(\theta)}_2^2 \leq \lambda_2\, \norm{\theta - \xi}_2^2.
    \end{equation}
    If $\norm{\theta(0) - \xi}_2 \leq r = \min \set{\delta_1, \delta_2 - \abs{\xi_2}}$, then for $0 < \eta \leq \lambda_1 / (2 \lambda_2)$, the updates $\theta(t + 1) = \theta(t) - \eta G_m(\theta(t))$ satisfy
    \begin{align*}
        \norm{\theta(t) - \xi}_2^2 \leq \left(1 - \frac{\eta\, \lambda_1}{2} \right)^t\, \norm{\theta(0) - \xi}_2^2.
    \end{align*}
    The differential equation $\theta^\prime = -G_m(\theta)$ has a solution defined for all $t \geq 0$, and every solution extends to all $t \geq 0$ and satisfies
    \begin{align*}
        \norm{\theta(t) - \xi}_2^2 \leq \norm{\theta(0) - \xi}_2^2\, e^{-\lambda_1 t}.
    \end{align*}
    Thus, for both gradient flow and gradient descent, the parameters remain in $\overline B(\xi, r)$ and converge to $\xi$.
\end{restatable}

We give the proof in \Cref{app:proof-sim-filtered-general}. For JFB, $p_- = p_+ = 1$; for two terms, $p_- = 1 - q$ and $p_+ = 1 + q$. More generally,
\begin{align*}
    p_- \geq \frac{1 - q^m}{1 + q}, \qquad p_+ \leq \frac{1 - q^m}{1 - q}.
\end{align*}

\subsection{Estimates for single-index models} \label{app:single-index-bounds}

\begin{lemma} \label{lem:single-index-bounds}
    Suppose the assumptions of \Cref{thm:nonlinear-convergence} hold, and let
    $S$ be as in \Cref{thm:nonlinear-gd-convergence-general}. Put $q = L \delta_2$.
    For this lemma, we also allow $m = \infty$, with
    $p_\infty(a) = (1 - a)^{-1}$ and $G_\infty = \grad_\theta R$.
    For $m \in \N \cup \{\infty\}$, define
    \begin{align*}
        p_- = \min_{\abs{a} \leq q} p_m(a), \qquad
        p_+ = \max_{\abs{a} \leq q} p_m(a).
    \end{align*}
    Then $S < \infty$, $p_- > 0$, and $G_m$ is continuous on
    $\{\theta : \abs{\theta_2} < 1 / L\}$. For every $\theta \in \Theta$,
    \begin{align*}
        \inner{G_m(\theta), \theta - \xi}
         & \geq \frac{2 \rho \gamma^2 p_-(1 - L \abs{\xi_2})}{(1 + q)^2}\,
        \norm{\theta - \xi}_2^2,                                           \\
        \norm{G_m(\theta)}_2^2
         & \leq \frac{4 L^4 p_+^2 S^2}{(1 - L \abs{\xi_2})^2}\,
        \norm{\theta - \xi}_2^2.
    \end{align*}
    Moreover,
    \begin{equation} \label{eq:sim-risk-upper}
        R(\theta, \xi) \leq \frac{L^2 S}{(1 - L \abs{\xi_2})^2}\,
        \norm{\theta - \xi}_2^2.
    \end{equation}
\end{lemma}

\begin{proof}
    Our proof is inspired by \citet{yehudai2020learning}. We write
    $z_\theta(x) = (x, y_\theta(x))^\top$ and
    $\omega_\theta(x) = \theta_1^\top x + \theta_2\, y_\theta(x)$.
    First, the fixed-point equation and Lipschitz continuity give
    \begin{align*}
        \abs{y_\theta(x)}
        \leq \abs{\sigma(0)} + L \norm{\theta_1}_2\, \norm{x}_2
        + L \abs{\theta_2}\, \abs{y_\theta(x)}.
    \end{align*}
    Since $\abs{\theta_2} \leq \delta_2 < 1 / L$, rearranging yields
    \begin{equation} \label{eq:y-bound-first}
        \abs{y_\theta(x)} \leq \frac{\abs{\sigma(0)} + L(\norm{\xi_1}_2 + \delta_1)\, \norm{x}_2}{1 - L \delta_2}, \qquad \theta \in \Theta.
    \end{equation}
    Thus $S = \sup_{\theta \in \Theta} \E \norm{z_\theta(X)}_2^2 < \infty$. On $\set{\norm{x}_2 \leq \alpha}$, we have $\abs{y_\theta(x)} \leq M$ and $\abs{\omega_\theta(x)} \leq \alpha(\norm{\xi_1}_2 + \delta_1) + M \delta_2$. The implicit function theorem gives
    \begin{align*}
        \grad_\theta y_\theta(x) = \frac{\sigma^\prime(\omega_\theta(x))}{1 - \theta_2\, \sigma^\prime(\omega_\theta(x))}\, z_\theta(x).
    \end{align*}
    On any compact subset of $\set{\abs{\theta_2} < 1 / L}$, this derivative and $z_\theta(x)$ are bounded by $C(1 + \norm{x}_2)$. Dominated convergence with $X \in L^2(\P)$ justifies differentiating the risk and gives $G_\infty = \grad_\theta R$ and continuity of every $G_m$.

    We fix $\theta \in \Theta$, write $e = \theta - \xi$, and define the secant slope
    \begin{align*}
        c(x) =
        \begin{cases}
            \dfrac{\sigma(\omega_\theta(x)) - \sigma(\omega_\xi(x))}
            {\omega_\theta(x) - \omega_\xi(x)}, & \omega_\theta(x) \neq \omega_\xi(x), \\[5pt]
            \sigma^\prime(\omega_\xi(x)),       & \omega_\theta(x) = \omega_\xi(x).
        \end{cases}
    \end{align*}
    Then $0 \leq c(x) \leq L$ everywhere, and $c(x) \geq \gamma$ on $\set{\norm{x}_2 \leq \alpha}$. By definition of $c$,
    \begin{align*}
        y_\theta - y_\xi
        = c(\omega_\theta - \omega_\xi)
        = c\bigl(e^\top z_\xi + \theta_2(y_\theta - y_\xi)\bigr).
    \end{align*}
    Rearranging gives the first identity below. Substituting it into
    $e^\top z_\theta = e^\top z_\xi + (\theta_2 - \xi_2)(y_\theta - y_\xi)$
    gives the second. For the third, we use
    $\omega_\theta - \omega_\xi = e^\top z_\theta + \xi_2(y_\theta - y_\xi)$
    instead. Thus
    \begin{align}
        y_\theta - y_\xi & = \frac{c}{1 - \theta_2 c}\, e^\top z_\xi, \label{eq:secant-target}              \\
        e^\top z_\theta  & = \frac{1 - \xi_2 c}{1 - \theta_2 c}\, e^\top z_\xi, \label{eq:feature-transfer} \\
        y_\theta - y_\xi & = \frac{c}{1 - \xi_2 c}\, e^\top z_\theta.\label{eq:secant-current}
    \end{align}
    Since $0 \leq c \leq L$ and $\abs{\theta_2}, \abs{\xi_2} < 1 / L$, the denominators are positive. Substituting into \eqref{eq:filtered-field} gives
    \begin{equation} \label{eq:radial-exact}
        \inner{G_m(\theta), e} = 2 \E \left[ \frac{\sigma^\prime(\omega_\theta)\, p_m(\theta_2\, \sigma^\prime(\omega_\theta))\, c(1 - \xi_2 c)}{(1 - \theta_2 c)^2}\, (e^\top z_\xi)^2 \right].
    \end{equation}
    Since $q < 1$, we have $p_m(a) = (1 - a^m) / (1 - a) > 0$ on $[-q, q]$ for finite $m$, and $p_\infty(a) = (1 - a)^{-1} > 0$ there as well. Thus the integrand in \eqref{eq:radial-exact} is nonnegative, and we can lower bound the expectation by restricting to $\{\norm{X}_2 \leq \alpha\}$. On this event, $\sigma^\prime(\omega_\theta)$ and $c$ are at least $\gamma$, while
    $1 - \xi_2 c \geq 1 - L \abs{\xi_2}$ and $1 - \theta_2 c \leq 1 + q$.
    It follows that
    \begin{align*}
        2 \inner{G_m(\theta), e} \geq \frac{4 \gamma^2 p_-(1 - L \abs{\xi_2})}{(1 + q)^2} e^\top \Sigma_\xi e \geq \frac{4 \rho \gamma^2 p_-(1 - L \abs{\xi_2})}{(1 + q)^2}\, \norm{e}_2^2.
    \end{align*}

    For the upper bound, \eqref{eq:secant-current} gives
    \begin{align*}
        \abs{y_\theta(x) - y_\xi(x)} \leq \frac{L}{1 - L \abs{\xi_2}} \norm{e}_2\, \norm{z_\theta(x)}_2.
    \end{align*}
    Consequently,
    \begin{align*}
        \norm{G_m(\theta)}_2 \leq 2 L p_+\, \E[\abs{y_\theta - y_\xi} \norm{z_\theta}_2] \leq \frac{2 L^2 p_+ S}{1 - L \abs{\xi_2}} \norm{e}_2.
    \end{align*}
    It follows that
    \begin{align*}
        \norm{G_m(\theta)}_2^2
        \leq \frac{4 L^4 p_+^2 S^2}{(1 - L \abs{\xi_2})^2}\,
        \norm{\theta - \xi}_2^2.
    \end{align*}
    Using the bound on $\abs{y_\theta(x) - y_\xi(x)}$ once more, we obtain
    \begin{align*}
        R(\theta, \xi)
         & = \E[(y_\theta(X) - y_\xi(X))^2]                  \\
         & \leq \frac{L^2}{(1 - L \abs{\xi_2})^2}\,
        \norm{\theta - \xi}_2^2\, \E[\norm{z_\theta(X)}_2^2] \\
         & \leq \frac{L^2 S}{(1 - L \abs{\xi_2})^2}\,
        \norm{\theta - \xi}_2^2. \qedhere
    \end{align*}
\end{proof}

\subsection[Proof of general single-index gradient-flow convergence]{Proof of \Cref{thm:nonlinear-convergence}} \label{app:proof-nonlinear-convergence}

\nonlinearflowgeneral*

\begin{proof}
    We apply \Cref{lem:single-index-bounds} with $m = \infty$.
    Since $p_\infty(a) = (1 - a)^{-1}$, its minimum on $[-L \delta_2, L \delta_2]$
    is $(1 + L \delta_2)^{-1}$. The lemma therefore gives
    \begin{align*}
        2 \inner{\grad_\theta R(\theta, \xi), \theta - \xi}
        \geq \lambda_1\, \norm{\theta - \xi}_2^2,
        \qquad \theta \in \Theta.
    \end{align*}
    The gradient $\grad_\theta R$ is continuous on $\{\theta : \abs{\theta_2} < 1 / L\}$, so a local solution exists by Peano's theorem. Along any portion of a solution in $\overline B(\xi, r)$,
    \begin{align*}
        \frac{d}{dt} \norm{\theta(t) - \xi}_2^2 = -2 \inner{\grad_\theta R(\theta(t), \xi), \theta(t) - \xi} \leq - \lambda_1\, \norm{\theta(t) - \xi}_2^2.
    \end{align*}
    At a boundary point, the derivative is at most $-\lambda_1 r^2 < 0$. A solution starting on the boundary therefore enters the ball, and a solution starting inside cannot first reach its boundary, since its squared distance would be decreasing near that first contact. Thus the solution remains in the ball. Multiplying the differential inequality by $e^{\lambda_1 t}$ gives
    \begin{align*}
        \frac{d}{dt}\left(e^{\lambda_1 t}\, \norm{\theta(t) - \xi}_2^2\right) \leq 0,
    \end{align*}
    and integrating gives the stated bound.

    To see that the solution exists for all nonnegative times, suppose it were defined only up to a finite time $T$. The continuous function $\grad_\theta R$ is bounded on $\overline B(\xi, r)$, so
    $\norm{\theta(t) - \theta(s)}_2 \leq C \abs{t - s}$ for $s,t < T$.
    Consequently $\theta(t)$ has a limit as $t \uparrow T$. This limit lies in the same closed ball, where $\abs{\theta_2} \leq \delta_2 < 1 / L$, and local existence at the limit extends the solution beyond $T$. This is a contradiction.
    Finally, if $\sigma \in C^2$ and $X$ is bounded, we can differentiate the implicit equation twice.
    On compact subsets of $\{\theta : \abs{\theta_2} < 1 / L\}$, these derivatives are bounded uniformly over $X$.
    Differentiating under the expectation shows that $\grad_\theta R$ is locally Lipschitz,
    so the gradient-flow solution is unique.
\end{proof}

\subsection[Proof of general single-index gradient-descent convergence]{Proof of \Cref{thm:nonlinear-gd-convergence-general}} \label{app:proof-nonlinear-gd-general}

\nonlineardescentgeneral*

\begin{proof}
    For exact backpropagation, the minimum and maximum of $p_\infty$ on
    $[-L \delta_2, L \delta_2]$ are $(1 + L \delta_2)^{-1}$ and $(1 - L \delta_2)^{-1}$.
    By \Cref{lem:single-index-bounds}, we have
    \begin{align*}
        2 \inner{\grad_\theta R(\theta, \xi), \theta - \xi}
         & \geq \lambda_1\, \norm{\theta - \xi}_2^2, \\
        \norm{\grad_\theta R(\theta, \xi)}_2^2
         & \leq \lambda_2\, \norm{\theta - \xi}_2^2
    \end{align*}
    for all $\theta \in \Theta$.
    For a discrete update,
    \begin{align*}
        \norm{\theta - \eta \grad_\theta R(\theta, \xi) - \xi}_2^2
         & = \norm{\theta - \xi}_2^2 - 2 \eta \inner{\grad_\theta R(\theta, \xi), \theta - \xi} + \eta^2\, \norm{\grad_\theta R(\theta, \xi)}_2^2 \\
         & \leq (1 - \eta\, \lambda_1 + \eta^2 \lambda_2)\, \norm{\theta - \xi}_2^2
        \leq (1 - \eta\, \lambda_1 / 2)\, \norm{\theta - \xi}_2^2.
    \end{align*}
    Applying Cauchy--Schwarz to the two bounds at any $\theta \neq \xi$ in $\overline B(\xi, r)$ gives
    $\lambda_1 \norm{\theta - \xi}_2^2 / 2
        \leq \sqrt{\lambda_2}\, \norm{\theta - \xi}_2^2$.
    Hence $\lambda_1 \leq 2 \sqrt{\lambda_2}$ and
    $0 \leq 1 - \eta \lambda_1 / 2 < 1$ for the prescribed step size.
    If $\theta(t) \in \overline B(\xi, r)$, the update inequality gives
    $\norm{\theta(t + 1) - \xi}_2 \leq \norm{\theta(t) - \xi}_2 \leq r$.
    Starting with the assumed initialization, induction therefore proves both that every iterate remains in this ball and that
    \begin{align*}
        \norm{\theta(t) - \xi}_2^2
         & \leq \left(1 - \frac{\eta \lambda_1}{2}\right)^t\, \norm{\theta(0) - \xi}_2^2. \qedhere
    \end{align*}
\end{proof}

\subsection[Proof of general convergence with fixed backward expansions]{Proof of \Cref{thm:sim-filtered-general}} \label{app:proof-sim-filtered-general}

\simfilteredgeneral*

\begin{proof}
    For finite $m$, \Cref{lem:single-index-bounds} gives \eqref{eq:target-directed-bounds}
    with the constants $\lambda_1$ and $\lambda_2$ in the statement.
    We first consider the discrete updates. For $\theta \in \Theta$,
    \begin{align*}
        \norm{\theta - \eta G_m(\theta) - \xi}_2^2
         & = \norm{\theta - \xi}_2^2 - 2 \eta \inner{G_m(\theta), \theta - \xi} + \eta^2\, \norm{G_m(\theta)}_2^2 \\
         & \leq (1 - \eta\, \lambda_1 + \eta^2 \lambda_2)\, \norm{\theta - \xi}_2^2
        \leq (1 - \eta\, \lambda_1 / 2)\, \norm{\theta - \xi}_2^2.
    \end{align*}
    Applying Cauchy--Schwarz to the two bounds at any $\theta \neq \xi$ in $\overline B(\xi, r)$ gives
    $\lambda_1 \norm{\theta - \xi}_2^2 / 2
        \leq \sqrt{\lambda_2}\, \norm{\theta - \xi}_2^2$.
    Hence $\lambda_1 \leq 2 \sqrt{\lambda_2}$ and
    $0 \leq 1 - \eta \lambda_1 / 2 < 1$ for the prescribed step size.
    If $\theta(t) \in \overline B(\xi, r)$, the update inequality gives
    $\norm{\theta(t + 1) - \xi}_2 \leq \norm{\theta(t) - \xi}_2 \leq r$.
    Starting with the assumed initialization, induction therefore proves both that every iterate remains in this ball and that
    \begin{align*}
        \norm{\theta(t) - \xi}_2^2
        \leq \left(1 - \frac{\eta \lambda_1}{2}\right)^t\, \norm{\theta(0) - \xi}_2^2.
    \end{align*}

    For continuous time, $G_m$ is continuous on $\{\theta : \abs{\theta_2} < 1 / L\}$, so a local solution exists by Peano's theorem. Along any portion of a solution in $\overline B(\xi, r)$,
    \begin{align*}
        \frac{d}{dt} \norm{\theta(t) - \xi}_2^2 = -2 \inner{G_m(\theta(t)), \theta(t) - \xi} \leq - \lambda_1\, \norm{\theta(t) - \xi}_2^2.
    \end{align*}
    At a boundary point, the derivative is at most $-\lambda_1 r^2 < 0$. A solution starting on the boundary therefore enters the ball, and a solution starting inside cannot first reach its boundary, since its squared distance would be decreasing near that first contact. Thus the solution remains in the ball. Multiplying the differential inequality by $e^{\lambda_1 t}$ gives
    \begin{align*}
        \frac{d}{dt}\left(e^{\lambda_1 t}\, \norm{\theta(t) - \xi}_2^2\right) \leq 0,
    \end{align*}
    and integrating gives the stated bound.

    To see that the solution exists for all nonnegative times, suppose it were defined only up to a finite time $T$. The continuous function $G_m$ is bounded on $\overline B(\xi, r)$, so
    $\norm{\theta(t) - \theta(s)}_2 \leq C \abs{t - s}$ for $s,t < T$.
    Consequently $\theta(t)$ has a limit as $t \uparrow T$. This limit lies in the same closed ball, where $\abs{\theta_2} \leq \delta_2 < 1 / L$, and local existence at the limit extends the solution beyond $T$. This is a contradiction.
\end{proof}

\subsection{Identifiability of the target}

\begin{lemma} \label{lem:target-identifiability}
    Under the assumptions of \Cref{thm:nonlinear-convergence-simplified}, the matrix
    $\Sigma_\xi$ in \eqref{eq:target-Gram} is positive definite for all sufficiently large $\alpha$.
\end{lemma}

\begin{proof}
    We choose $\alpha$ large enough that the open ball on which the density is positive is contained in $\set{\norm{x}_2 < \alpha}$. Suppose $(u, v)$ is in the nullspace of $\Sigma_\xi$. Continuity of $y_\xi$ and positivity of the density imply $u^\top x + v y_\xi(x) = 0$ everywhere on that ball. If $v = 0$, then $u = 0$. Otherwise, we put $\kappa = -u / v$, so the fixed-point equation gives
    \begin{align*}
        \sigma((\xi_1 + \xi_2 \kappa)^\top x) = \kappa^\top x
    \end{align*}
    on the ball. If $\xi_1 + \xi_2 \kappa \neq 0$, we choose an interior point $x_0$ and restrict to
    $x = x_0 + t(\xi_1 + \xi_2 \kappa)$ for sufficiently small $\abs{t}$. We obtain
    \begin{align*}
        \sigma\bigl((\xi_1 + \xi_2 \kappa)^\top x_0
        + t \norm{\xi_1 + \xi_2 \kappa}_2^2\bigr)
        = \kappa^\top x_0 + t \kappa^\top(\xi_1 + \xi_2 \kappa).
    \end{align*}
    The argument on the left varies over an open interval, on which this identity makes $\sigma$ affine, contrary to the assumption. If $\xi_1 + \xi_2 \kappa = 0$, the identity instead gives $\kappa^\top x = \sigma(0)$ on the ball. A linear function can be constant on a ball only if $\kappa = 0$. This would imply $\sigma(0) = 0$ and $\xi_1 = 0$, again contradicting the assumptions. Thus $\Sigma_\xi \succ 0$.
\end{proof}

\subsection[Proof of simplified single-index gradient-flow convergence]{Proof of \Cref{thm:nonlinear-convergence-simplified}} \label{app:proof-cor-nonlinear-convergence}

\nonlinearflowsimplified*

\begin{proof}
    We choose $r$ such that
    $\norm{\theta(0) - \xi}_2 < r < 1 / L - \abs{\xi_2}$ and set
    $\delta_1 = r$ and $\delta_2 = \abs{\xi_2} + r$.
    By \Cref{lem:target-identifiability}, we can choose $\alpha$ so that $\Sigma_\xi \succ 0$.
    Since $\sigma^\prime$ is continuous and strictly positive, it has a positive minimum on
    $[-\alpha(\norm{\xi_1}_2 + \delta_1) - M \delta_2,\,
                \alpha(\norm{\xi_1}_2 + \delta_1) + M \delta_2]$.
    Thus the assumptions and initialization of \Cref{thm:nonlinear-convergence} hold.
    Applying that theorem with $\lambda = \lambda_1$ gives
    \begin{align*}
        \norm{\theta(t) - \xi}_2^2
         & \leq \norm{\theta(0) - \xi}_2^2\, e^{-\lambda t}, \qquad t \geq 0. \qedhere
    \end{align*}
\end{proof}

\subsection[Proof of simplified single-index gradient-descent convergence]{Proof of \Cref{thm:nonlinear-gd-convergence}} \label{app:proof-nonlinear-gd-simplified}

\nonlineardescentsimplified*

\begin{proof}
    We choose $r$ such that
    $\norm{\theta(0) - \xi}_2 < r < 1 / L - \abs{\xi_2}$ and set
    $\delta_1 = r$ and $\delta_2 = \abs{\xi_2} + r$.
    By \Cref{lem:target-identifiability}, we can choose $\alpha$ so that $\Sigma_\xi \succ 0$.
    Since $\sigma^\prime$ is continuous and strictly positive, it has a positive minimum on
    $[-\alpha(\norm{\xi_1}_2 + \delta_1) - M \delta_2,\,
                \alpha(\norm{\xi_1}_2 + \delta_1) + M \delta_2]$.
    Thus the assumptions and initialization of \Cref{thm:nonlinear-convergence} hold.
    We apply \Cref{thm:nonlinear-gd-convergence-general} with $\lambda = \lambda_1 / 2$
    and $\eta_0 = \min\{\lambda_1 / (2 \lambda_2),\, 1 / \lambda_1\}$.
    For every $0 < \eta \leq \eta_0$, that theorem gives
    \begin{align*}
        \norm{\theta(t) - \xi}_2^2
        \leq (1 - \eta \lambda)^t\, \norm{\theta(0) - \xi}_2^2.
    \end{align*}
    Our choice of $\eta_0$ ensures that $1 - \eta \lambda \in [1 / 2, 1)$.
\end{proof}

\subsection[Proof of simplified convergence with fixed backward expansions]{Proof of \Cref{thm:sim-filtered}} \label{app:proof-sim-filtered-simplified}

\simfilteredsimplified*

\begin{proof}
    We choose $r$ such that
    $\norm{\theta(0) - \xi}_2 < r < 1 / L - \abs{\xi_2}$ and set
    $\delta_1 = r$ and $\delta_2 = \abs{\xi_2} + r$.
    By \Cref{lem:target-identifiability}, we can choose $\alpha$ so that $\Sigma_\xi \succ 0$.
    Since $\sigma^\prime$ is continuous and strictly positive, it has a positive minimum on
    $[-\alpha(\norm{\xi_1}_2 + \delta_1) - M \delta_2,\,
                \alpha(\norm{\xi_1}_2 + \delta_1) + M \delta_2]$.
    Thus the assumptions and initialization of \Cref{thm:nonlinear-convergence} hold.
    We apply \Cref{thm:sim-filtered-general} with $\lambda = \lambda_1$ and
    $\eta_0 = \min\{\lambda_1 / (2 \lambda_2),\, 1 / \lambda_1\}$, using the constants
    for the chosen $m$. This gives the asserted estimates for gradient flow and the discrete updates,
    with $1 - \eta \lambda / 2 \in [1 / 2, 1)$.
    In both cases the parameters remain in $\overline B(\xi, r)$, where
    $\abs{\theta_2} \leq \abs{\xi_2} + r < 1 / L$, and \eqref{eq:sim-risk-upper}
    implies $R(\theta(t), \xi) \to 0$.
\end{proof}

\subsection[Scalar linear backward approximations]{Proof of \Cref{lem:linear-time-change}} \label{app:proof-linear-rescaling}

\linearrescaling*

\begin{proof}
    The identity $(1 - a)\, p_m(a) = 1 - a^m$ gives $G_m(\theta) = (1 - \theta_2^m)\, \grad_\theta R(\theta, \xi)$. Along $\theta^\prime = -G_m(\theta)$, the chain rule gives
    \begin{equation} \label{eq:linear-trunc-risk-derivative}
        \frac{d}{dt} R(\theta(t), \xi)
        = -(1 - \theta_2(t)^m)\, \norm{\grad_\theta R(\theta(t), \xi)}_2^2.
    \end{equation}
    Differentiating
    $R(\theta, \xi) = (\theta_1 / (1 - \theta_2) - \xi)^\top
        \Sigma(\theta_1 / (1 - \theta_2) - \xi)$ gives
    \begin{align*}
        \grad_{\theta_1} R
         & = \frac{2}{1 - \theta_2}\, \Sigma\left(\frac{\theta_1}{1 - \theta_2} - \xi\right), \\
        \partial_{\theta_2} R
         & = \frac{2}{(1 - \theta_2)^2}\, \theta_1^\top
        \Sigma\left(\frac{\theta_1}{1 - \theta_2} - \xi\right).
    \end{align*}
    In particular, $\theta_1^\top \grad_{\theta_1} R + (\theta_2 - 1)\, \partial_{\theta_2} R = 0$. Hence
    \begin{align*}
        w^\prime(t)
         & = 2 \theta_1^\top \theta_1^\prime + 2(\theta_2 - 1)\, \theta_2^\prime \\
         & = -2(1 - \theta_2^m)
        \left(\theta_1^\top \grad_{\theta_1} R + (\theta_2 - 1)\, \partial_{\theta_2} R\right) = 0,
    \end{align*}
    so $w(t) = w(0)$ for as long as the solution exists.
\end{proof}

\subsection[Scalar linear gradient flow with truncated backpropagation]{Proof of \Cref{thm:linear-trunc-flow}} \label{app:proof-linear-trunc}

\lineartruncflow*

\begin{proof}
    Define
    \begin{align*}
        R_0   & \coloneq R(\theta(0), \xi), \qquad
        K \coloneq \norm{\xi}_2 + \sqrt{\frac{R_0}{\lambda_\mathrm{min}(\Sigma)}}, \\
        \beta & \coloneq \sqrt{\frac{w(0)}{1 + K^2}}, \qquad
        c_m \coloneq \min_{1 - \sqrt{w(0)} \leq s \leq 1 - \beta} (1 - s^m).
    \end{align*}
    We write $\varphi = \theta_1 / (1 - \theta_2)$. For odd $m$, the factor $1 - \theta_2^m$ is positive whenever $\theta_2 < 1$. For even $m$, conservation and $w(0) < 4$ give
    $1 - \theta_2 \leq \sqrt{w(0)} < 2$, so $-1 < \theta_2 < 1$ and the factor is again positive. Thus, for as long as $\theta_2 < 1$, the risk is nonincreasing and
    \begin{align*}
        \norm{\varphi}_2 \leq \norm{\xi}_2 + \sqrt{\frac{R_0}{\lambda_\mathrm{min}(\Sigma)}} = K,
        \qquad 1 - \theta_2 = \sqrt{\frac{w(0)}{1 + \norm{\varphi}_2^2}} \geq \beta.
    \end{align*}
    Together with $\norm{\theta_1}_2 \leq \sqrt{w(0)}$, these bounds keep the parameters in a compact set on which the update is smooth. If the solution ended at a finite time, boundedness of its derivative would give a limiting parameter in this set, from which the solution could be continued. Therefore it exists for all $t \geq 0$ and remains in $\theta_2 < 1$. Moreover,
    \begin{align*}
        \norm{\grad_\theta R}_2^2
         & = \frac{4}{(1 - \theta_2)^2}
        \left(\norm{\Sigma(\varphi - \xi)}_2^2 + [\varphi^\top \Sigma(\varphi - \xi)]^2\right) \\
         & \geq \frac{4 \lambda_\mathrm{min}(\Sigma)}{w(0)}\, R.
    \end{align*}
    Since $1 - \sqrt{w(0)} \leq \theta_2 \leq 1 - \beta$, we have $1 - \theta_2^m \geq c_m > 0$. Substituting into \eqref{eq:linear-trunc-risk-derivative} and integrating gives
    \begin{align*}
        R^\prime & \leq -\frac{4 \lambda_\mathrm{min}(\Sigma)\, c_m}{w(0)}\, R,                        \\
        R(\theta(t), \xi)
                 & \leq R_0\, \exp\left(-\frac{4 \lambda_\mathrm{min}(\Sigma)\, c_m}{w(0)}\, t\right).
    \end{align*}
    Because $\lambda_\mathrm{min}(\Sigma)\, \norm{\varphi(t) - \xi}_2^2 \leq R(\theta(t), \xi)$, it follows that $\varphi(t) \to \xi$. Finally,
    \begin{align*}
        1 - \theta_2(t)
         & = \sqrt{\frac{w(0)}{1 + \norm{\varphi(t)}_2^2}}
        \longrightarrow \sqrt{\frac{w(0)}{1 + \norm{\xi}_2^2}}, \\
        \theta_1(t)
         & = (1 - \theta_2(t))\, \varphi(t)
        \longrightarrow \sqrt{\frac{w(0)}{1 + \norm{\xi}_2^2}}\, \xi. \qedhere
    \end{align*}
\end{proof}

\subsection[Scalar linear gradient descent with truncated backpropagation]{Proof of \Cref{thm:linear-trunc-descent}} \label{app:proof-linear-trunc-descent}

\lineartruncdescent*

\begin{proof}
    We use $R_0$, $K$, $\beta$, and $\varphi$ from \Cref{app:proof-linear-trunc}, and write $h(s) = 1 - s^m$.
    We set $\theta_- = -\infty$ when $m$ is odd and $\theta_- = -1$ when $m$ is even.
    In either case, $h > 0$ on $(\theta_-, 1)$ and the initialization satisfies
    $\sqrt{w(0)} < 1 - \theta_-$. If $R_0 = 0$, all iterates are stationary, so we may assume $R_0 > 0$.
    We write $R_t \coloneq R(\theta(t), \xi)$ and $\grad R_t \coloneq \grad_\theta R(\theta(t), \xi)$.
    We choose $\sqrt{w(0)} < B < 1 - \theta_-$ and
    \begin{align*}
        0 < \delta \leq \beta / 2, \qquad B + \delta < 1 - \theta_-, \qquad \mc{K} = \set{\theta : \norm{\theta_1}_2 \leq B + \delta, \ \beta / 2 \leq 1 - \theta_2 \leq B + \delta}.
    \end{align*}
    When $\theta_- = -\infty$, the second restriction is vacuous. On this compact convex set, let
    \begin{align*}
        h_- = \min_{\mc{K}} h(\theta_2) > 0, \quad h_+ = \max_{\mc{K}} h(\theta_2), \quad L_{\mc{K}} = \sup_{\mc{K}} \norm{\grad^2 R}_\mathrm{op}.
    \end{align*}
    We choose $\eta_0 > 0$ so small that
    \begin{align*}
        \eta_0 \sup_{\mc{K}} h(\theta_2)\, \norm{\grad R}_2 \leq \delta, \quad \eta_0 L_{\mc{K}} h_+ \leq 1, \quad 2 \eta_0 h_+ R_0 \leq B^2 - w(0), \quad \frac{2 \eta_0 \lambda_\mathrm{min}(\Sigma)\, h_-}{B^2} \leq \frac{1}{2}.
    \end{align*}
    All suprema are finite and $B^2 > w(0)$, so such a choice exists. We fix $0 < \eta \leq \eta_0$.
    We prove by induction that $R_t \leq R_0$, $w(0) \leq w(t) \leq B^2$, and
    $\beta \leq 1 - \theta_2(t) \leq B$. These inequalities hold at initialization.
    Suppose they hold through time $t$. The update length is at most $\delta$, and
    $\norm{\theta_1(t)}_2 \leq B$, so the entire update segment lies in $\mc{K}$. Taylor's theorem gives
    \begin{align*}
        R_{t + 1}
         & \leq R_t - \eta h(\theta_2(t))\, \norm{\grad R_t}_2^2
        + \frac{L_{\mc{K}} \eta^2 h(\theta_2(t))^2}{2}\, \norm{\grad R_t}_2^2 \\
         & \leq R_t - \frac{\eta}{2}\, h(\theta_2(t))\, \norm{\grad R_t}_2^2.
    \end{align*}
    Since $\theta_1^\top \grad_{\theta_1} R + (\theta_2 - 1)\, \partial_{\theta_2} R = 0$, expanding $w(t + 1)$ gives
    \begin{equation} \label{eq:linear-discrete-radius}
        w(t + 1) = w(t) + \eta^2 h(\theta_2(t))^2\, \norm{\grad R_t}_2^2.
    \end{equation}
    Combining these inequalities and summing gives
    \begin{align*}
        0 \leq w(t + 1) - w(0)
         & \leq 2 \eta h_+ \sum_{j = 0}^{t} (R_j - R_{j + 1})                  \\
         & = 2 \eta h_+ (R_0 - R_{t + 1}) \leq 2 \eta h_+ R_0 \leq B^2 - w(0).
    \end{align*}
    Also, $R_{t + 1} \leq R_0$ implies $\norm{\varphi(t + 1)}_2 \leq K$.
    Since the endpoint lies in $\mc{K}$, we have $1 - \theta_2(t + 1) > 0$ and
    \begin{align*}
        1 - \theta_2(t + 1)
        = \sqrt{\frac{w(t + 1)}{1 + \norm{\varphi(t + 1)}_2^2}} \geq \beta.
    \end{align*}
    The upper bound $1 - \theta_2(t + 1) \leq B$ follows from $w(t + 1) \leq B^2$, completing the induction.
    The gradient formula in \Cref{app:proof-linear-trunc} gives
    $\norm{\grad R_t}_2^2 \geq 4 \lambda_\mathrm{min}(\Sigma)\, R_t / B^2$ at every iterate. Therefore
    \begin{align*}
        R_t \leq \left(1 - \frac{2 \eta\, \lambda_\mathrm{min}(\Sigma)\, h_-}{B^2} \right)^t\, R_0.
    \end{align*}
    Our choice of $\eta_0$ puts the factor in parentheses in $[1 / 2, 1)$.
    Thus $R_t \to 0$, and the inequality
    $\lambda_\mathrm{min}(\Sigma)\, \norm{\varphi(t) - \xi}_2^2 \leq R_t$
    gives $\varphi(t) \to \xi$. Moreover, $w(t)$ is increasing and bounded, so it has a limit
    $w_\infty \geq w(0) > 0$. Defining $\alpha_\eta = \sqrt{w_\infty / (1 + \norm{\xi}_2^2)}$, we obtain
    \begin{align*}
        1 - \theta_2(t)
                    & = \sqrt{\frac{w(t)}{1 + \norm{\varphi(t)}_2^2}} \longrightarrow \alpha_\eta, \\
        \theta_1(t) & = (1 - \theta_2(t))\, \varphi(t) \longrightarrow \alpha_\eta \xi. \qedhere
    \end{align*}
\end{proof}

\subsection[The even-truncation counterexample]{Proof of \Cref{prop:even-stall}} \label{app:proof-even-stall}

\evenstall*

\begin{proof}
    Writing $\varphi = \theta_1 / (1 - \theta_2)$ as above, the flow in this one-dimensional example satisfies
    \begin{align*}
        \theta_2^\prime = -2 p_m(\theta_2)\, \varphi(\varphi - 1), \qquad \varphi^\prime = -\frac{2 p_m(\theta_2)}{1 - \theta_2} (1 + \varphi^2)(\varphi - 1), \qquad \theta_1^2 + (1 - \theta_2)^2 = 9.
    \end{align*}
    For even $m$,
    \begin{align*}
        p_m(\theta_2) = (1 + \theta_2) \sum_{j = 0}^{m / 2 - 1} \theta_2^{2j}.
    \end{align*}
    On $-1 < \theta_2 \leq 0$, the sum is positive and bounded. Conservation gives
    $\theta_1^2 = 9 - (1 - \theta_2)^2 \geq 5$ on this interval. Since $\theta_1(0) > 0$, continuity implies $\theta_1(t) > 0$ for as long as $\theta_2$ remains in this interval. Therefore
    \begin{align*}
        \varphi = \sqrt{\frac{9}{(1 - \theta_2)^2} - 1} \geq \frac{\sqrt{5}}{2} > 1.
    \end{align*}
    Consequently $\theta_2$ decreases and $\varphi$ decreases while the trajectory is in this region. Moreover,
    \begin{align*}
        \frac{d}{dt} (1 + \theta_2) = -2(1 + \theta_2) \left(\sum_{j = 0}^{m / 2 - 1} \theta_2^{2j} \right) \varphi(\varphi - 1).
    \end{align*}
    Integrating this scalar equation gives
    \begin{align*}
        1 + \theta_2(t)
        = \exp\left(-2 \int_0^t
        \left(\sum_{j = 0}^{m / 2 - 1} \theta_2(s)^{2j}\right)
        \varphi(s)(\varphi(s) - 1)\, ds\right).
    \end{align*}
    The integrand is bounded while $-1 < \theta_2 \leq 0$, so the right side remains strictly positive at every finite time. Thus $\theta_2$ cannot reach $-1$ in finite time. The trajectory remains in the compact set
    $-1 \leq \theta_2 \leq 0$, $\sqrt{5} \leq \theta_1 \leq \sqrt{8}$,
    on which the vector field is smooth, and the solution extends to all $t \geq 0$.
    Since $\theta_2$ is decreasing and bounded below, it has a limit. If this limit were greater than $-1$, the formula for $\theta_2^\prime$ and the bound $\varphi \geq \sqrt{5} / 2$ would give
    $\theta_2^\prime \leq -c$ for some $c > 0$ at all sufficiently large times.
    Integration would then force $\theta_2(t) \to -\infty$, a contradiction.
    Hence $\theta_2(t) \to -1$, and
    $\theta_1(t) = \sqrt{9 - (1 - \theta_2(t))^2} \to \sqrt{5}$.
    At this limit the risk is $(\sqrt{5} / 2 - 1)^2 > 0$ and
    $\partial_{\theta_1} R = \sqrt{5} / 2 - 1 > 0$, although $p_m(-1) = 0$.

    For the discrete two-term update, direct algebra gives
    \begin{align*}
        1 + \theta_2(t + 1)
         & = (1 + \theta_2(t)) \set{1 - 2 \eta \varphi(t)(\varphi(t) - 1)},                                          \\
        \varphi(t + 1) - 1
         & = (\varphi(t) - 1) \frac{1 - \theta_2(t) - 2 \eta(1 + \theta_2(t))(1 + \varphi(t))}{1 - \theta_2(t + 1)}, \\
        \varphi(t + 1) - \varphi(t)
         & = -\frac{2 \eta(1 + \theta_2(t))(1 + \varphi(t)^2)(\varphi(t) - 1)}{1 - \theta_2(t + 1)}.
    \end{align*}
    Suppose $-1 < \theta_2(t) \leq 0$ and $1 < \varphi(t) \leq \sqrt{8}$, as holds initially. The prescribed step size gives
    \begin{align*}
        1 - 2 \eta \varphi(t)(\varphi(t) - 1)
         & \geq 1 - 2 \eta \sqrt{8}(\sqrt{8} - 1) > 0, \\
        1 - \theta_2(t) - 2 \eta(1 + \theta_2(t))(1 + \varphi(t))
         & \geq 1 - 2 \eta(1 + \sqrt{8}) > 0.
    \end{align*}
    The first update identity therefore gives $-1 < \theta_2(t + 1) < \theta_2(t)$.
    The second gives $\varphi(t + 1) > 1$, and the third gives
    $\varphi(t + 1) < \varphi(t)$. Induction proves these inequalities for all $t$.
    By~\eqref{eq:linear-discrete-radius}, $w(t) \geq 9$, so
    \begin{align*}
        \varphi(t)
        = \sqrt{\frac{w(t)}{(1 - \theta_2(t))^2} - 1}
        \geq \frac{\sqrt{5}}{2}.
    \end{align*}
    Both $\theta_2(t)$ and $\varphi(t)$ are decreasing and bounded below, so they converge.
    If $\lim_{t \to \infty} \theta_2(t) > -1$, then
    $\theta_2(t) - \theta_2(t + 1) = 2 \eta(1 + \theta_2(t))\, \varphi(t)(\varphi(t) - 1)$
    would have a strictly positive limit, contradicting convergence.
    Thus $\theta_2(t) \to -1$, $\theta_1(t) = (1 - \theta_2(t))\, \varphi(t)$ also converges, and
    $\lim_t R(\theta(t), \xi) = (\lim_t \varphi(t) - 1)^2 \geq (\sqrt{5} / 2 - 1)^2$.
    For JFB, we apply \Cref{thm:linear-trunc-flow,thm:linear-trunc-descent} with $m = 1$ to conclude that both gradient flow and gradient descent converge to zero risk.
\end{proof}

\subsection[Damping the backward approximation]{Proof of \Cref{cor:linear-damped}} \label{app:proof-damping}

\lineardamped*

\begin{proof}
    Since $\lambda > 0$, the equation
    \begin{align*}
        y = (1 - \lambda)y + \lambda(\theta_1^\top x + \theta_2 y)
    \end{align*}
    is equivalent to $y = \theta_1^\top x + \theta_2 y$, so damping does not change the equilibrium. The damped map has state derivative $1 - \lambda + \lambda\, \theta_2$ and parameter derivative $\lambda(x, y_\theta(x))^\top$. Thus
    \begin{align*}
        (1 - \theta_2) \lambda \sum_{j = 0}^{m - 1}(1 - \lambda + \lambda\, \theta_2)^j = 1 - (1 - \lambda + \lambda\, \theta_2)^m,
    \end{align*}
    and the approximate update is $h(\theta_2)\, \grad_\theta R(\theta, \xi)$, where
    \begin{align*}
        h(s) \coloneq 1 - (1 - \lambda + \lambda s)^m > 0,
        \qquad 1 - 2 / \lambda < s < 1.
    \end{align*}
    We write $\varphi = \theta_1 / (1 - \theta_2)$ and define
    \begin{align*}
        R_0 & \coloneq R(\theta(0), \xi), \qquad
        K \coloneq \norm{\xi}_2 + \sqrt{\frac{R_0}{\lambda_\mathrm{min}(\Sigma)}}, \qquad
        \beta \coloneq \sqrt{\frac{w(0)}{1 + K^2}}.
    \end{align*}
    If $R_0 = 0$, the gradient vanishes and both updates leave the parameters unchanged. We may therefore assume $R_0 > 0$.

    We first consider gradient flow. The identity
    $(\theta_1, \theta_2 - 1)^\top \grad_\theta R = 0$ gives
    \begin{align*}
        w^\prime(t) = -2 h(\theta_2(t))\,
        (\theta_1(t), \theta_2(t) - 1)^\top \grad_\theta R(\theta(t), \xi) = 0.
    \end{align*}
    As long as $\theta_2 < 1$, we consequently have
    $0 < 1 - \theta_2 \leq \sqrt{w(0)} < 2 / \lambda$. Since $h(\theta_2) > 0$ on this interval,
    \begin{align*}
        \frac{d}{dt} R(\theta(t), \xi)
        = -h(\theta_2(t))\, \norm{\grad_\theta R(\theta(t), \xi)}_2^2 \leq 0.
    \end{align*}
    Moreover, $R(\theta, \xi) = (\varphi - \xi)^\top \Sigma(\varphi - \xi)$, so
    \begin{align*}
        \norm{\varphi}_2 \leq K, \qquad
        1 - \theta_2 = \sqrt{\frac{w(0)}{1 + \norm{\varphi}_2^2}} \geq \beta.
    \end{align*}
    Thus the solution remains in the compact set
    $\{\theta : \norm{\theta_1}_2 \leq \sqrt{w(0)},\
        \beta \leq 1 - \theta_2 \leq \sqrt{w(0)}\}$, on which the update is smooth. In particular, the solution cannot reach $\theta_2 = 1$ or become unbounded in finite time, and it extends to all $t \geq 0$. On this set,
    \begin{align*}
        \abs{1 - \lambda + \lambda\, \theta_2} \leq q_\lambda \coloneq \max \set{\abs{1 - \lambda \beta}, \abs{1 - \lambda \sqrt{w(0)}}} < 1,
    \end{align*}
    and hence $h(\theta_2) \geq 1 - q_\lambda^m > 0$. Direct differentiation of the risk gives
    \begin{align*}
        \norm{\grad_\theta R}_2^2
         & = \frac{4}{(1 - \theta_2)^2}
        \left(\norm{\Sigma(\varphi - \xi)}_2^2
        + [\varphi^\top \Sigma(\varphi - \xi)]^2\right)          \\
         & \geq \frac{4 \lambda_\mathrm{min}(\Sigma)}{w(0)}\, R.
    \end{align*}
    Therefore
    \begin{align*}
        \frac{d}{dt} R(\theta(t), \xi)
         & \leq -\frac{4 \lambda_\mathrm{min}(\Sigma)(1 - q_\lambda^m)}{w(0)}\, R(\theta(t), \xi),        \\
        R(\theta(t), \xi)
         & \leq R_0\, \exp\left(-\frac{4 \lambda_\mathrm{min}(\Sigma)(1 - q_\lambda^m)}{w(0)}\, t\right).
    \end{align*}
    The second inequality follows by multiplying the first by the corresponding positive exponential and integrating from $0$ to $t$. Since
    $\lambda_\mathrm{min}(\Sigma)\, \norm{\varphi(t) - \xi}_2^2 \leq R(\theta(t), \xi)$,
    we have $\varphi(t) \to \xi$. Conservation of $w$ then gives
    \begin{align*}
        1 - \theta_2(t) & \longrightarrow
        \sqrt{\frac{w(0)}{1 + \norm{\xi}_2^2}},                                         \\
        \theta_1(t) = (1 - \theta_2(t))\, \varphi(t)
                        & \longrightarrow \sqrt{\frac{w(0)}{1 + \norm{\xi}_2^2}}\, \xi.
    \end{align*}

    We next prove the claim for gradient descent. We choose $B$ and $\delta$ so that
    \begin{align*}
        \sqrt{w(0)} < B < 2 / \lambda, \qquad
        0 < \delta \leq \beta / 2, \qquad B + \delta < 2 / \lambda,
    \end{align*}
    and consider the compact convex set
    \begin{align*}
        \mc{K} \coloneq \set{\theta :
            \norm{\theta_1}_2 \leq B + \delta,\
            \beta / 2 \leq 1 - \theta_2 \leq B + \delta}.
    \end{align*}
    The risk is smooth on $\mc{K}$ and $h(\theta_2)$ is strictly positive there. Consequently, the constants
    \begin{align*}
        h_- \coloneq \min_{\mc{K}} h(\theta_2) > 0, \qquad
        h_+ \coloneq \max_{\mc{K}} h(\theta_2), \qquad
        L_{\mc{K}} \coloneq \max_{\mc{K}} \norm{\grad_\theta^2 R}_\mathrm{op}
    \end{align*}
    are finite. We choose $\eta_0 > 0$ sufficiently small that
    \begin{align*}
        \eta_0 \max_{\mc{K}} h(\theta_2)\, \norm{\grad_\theta R}_2 & \leq \delta,
                                                                   & \eta_0 L_{\mc{K}} h_+                                   & \leq 1,           \\
        2 \eta_0 h_+ R_0                                           & \leq B^2 - w(0),
                                                                   & \frac{2 \eta_0 \lambda_\mathrm{min}(\Sigma)\, h_-}{B^2} & \leq \frac{1}{2}.
    \end{align*}
    Such a choice is possible because $B^2 > w(0)$.

    We fix $0 < \eta \leq \eta_0$ and write $R_t = R(\theta(t), \xi)$. We prove by induction that
    \begin{align*}
        R_t \leq R_0, \qquad w(0) \leq w(t) \leq B^2,
        \qquad \beta \leq 1 - \theta_2(t) \leq B.
    \end{align*}
    These inequalities hold at $t = 0$. Suppose they hold through time $t$. Since $\norm{\theta_1(t)}_2 \leq B$, the bound on the update length ensures that the segment from $\theta(t)$ to $\theta(t + 1)$ lies in $\mc{K}$. Taylor's theorem on this segment gives
    \begin{align*}
        R_{t + 1}
         & \leq R_t - \eta h(\theta_2(t))\, \norm{\grad_\theta R(\theta(t), \xi)}_2^2 \\
         & \quad + \frac{L_{\mc{K}} \eta^2 h(\theta_2(t))^2}{2}\,
        \norm{\grad_\theta R(\theta(t), \xi)}_2^2                                     \\
         & \leq R_t - \frac{\eta}{2} h(\theta_2(t))\,
        \norm{\grad_\theta R(\theta(t), \xi)}_2^2.
    \end{align*}
    We expand $w(t + 1)$ and use
    $\theta_1^\top \grad_{\theta_1} R + (\theta_2 - 1)\, \partial_{\theta_2} R = 0$
    to obtain
    \begin{align*}
        w(t + 1) - w(t)
         & = \eta^2 h(\theta_2(t))^2\, \norm{\grad_\theta R(\theta(t), \xi)}_2^2 \\
         & \leq 2 \eta h_+ (R_t - R_{t + 1}).
    \end{align*}
    Summing from $0$ to $t$ yields
    \begin{align*}
        w(0) \leq w(t + 1)
        \leq w(0) + 2 \eta h_+ (R_0 - R_{t + 1})
        \leq w(0) + 2 \eta h_+ R_0 \leq B^2.
    \end{align*}
    Since $R_{t + 1} \leq R_0$, we also have $\norm{\varphi(t + 1)}_2 \leq K$. The endpoint lies in $\mc{K}$, so $1 - \theta_2(t + 1) > 0$ and
    \begin{align*}
        1 - \theta_2(t + 1)
        = \sqrt{\frac{w(t + 1)}{1 + \norm{\varphi(t + 1)}_2^2}}
        \geq \sqrt{\frac{w(0)}{1 + K^2}} = \beta.
    \end{align*}
    The upper bound $1 - \theta_2(t + 1) \leq B$ follows from $w(t + 1) \leq B^2$, completing the induction.

    At every iterate, the gradient formula above now gives
    \begin{align*}
        \norm{\grad_\theta R(\theta(t), \xi)}_2^2
        \geq \frac{4 \lambda_\mathrm{min}(\Sigma)}{B^2}\, R_t.
    \end{align*}
    Substituting into the descent inequality and iterating, we obtain
    \begin{align*}
        R_t \leq R_0 \left(1 - \frac{2 \eta \lambda_\mathrm{min}(\Sigma)\, h_-}{B^2}\right)^t.
    \end{align*}
    The factor in parentheses lies in $[1 / 2, 1)$. Taking $c$ to be the smaller of
    $4 \lambda_\mathrm{min}(\Sigma)(1 - q_\lambda^m) / w(0)$ and
    $2 \lambda_\mathrm{min}(\Sigma)\, h_- / B^2$ gives the two risk bounds in the statement.
    Finally, $w(t)$ is increasing and bounded, so it has a limit $w_\infty \geq w(0) > 0$. The risk bound implies $\varphi(t) \to \xi$, and consequently
    \begin{align*}
        1 - \theta_2(t)
         & = \sqrt{\frac{w(t)}{1 + \norm{\varphi(t)}_2^2}}
        \longrightarrow \sqrt{\frac{w_\infty}{1 + \norm{\xi}_2^2}}, \\
        \theta_1(t)
         & = (1 - \theta_2(t))\, \varphi(t)
        \longrightarrow \sqrt{\frac{w_\infty}{1 + \norm{\xi}_2^2}}\, \xi. \qedhere
    \end{align*}
\end{proof}

\subsection{The same linearized training dynamics} \label{sec:local}

We write $\omega_\theta(x) = \theta_1^\top x + \theta_2\, y_\theta(x)$, as in \Cref{app:single-index-bounds}.

\begin{restatable}[Linearization at a realizable target]{proposition}{locallinearization} \label{thm:local}
    Suppose $X \in L^2(\P)$, $\sigma \in C^1(\R)$, $0 \leq \sigma^\prime \leq L$ for some $L > 0$, and $L \abs{\xi_2} < 1$. Put
    \begin{align*}
        \grad y_\xi(x) \coloneq \left.\grad_\theta y_\theta(x) \right|_{\theta = \xi}
        = \frac{\sigma^\prime(\omega_\xi(x))}{1 - \xi_2\, \sigma^\prime(\omega_\xi(x))}
        \begin{pmatrix} x \\ y_\xi(x) \end{pmatrix}.
    \end{align*}
    Each finite backward direction is differentiable at $\xi$, and
    \begin{align}
        \grad_\theta^2 R(\xi, \xi) & = 2 \E[\grad y_\xi(X) \grad y_\xi(X)^\top], \label{eq:H} \\
        DG_m(\xi)                  & = 2 \E \left[
            \set{1 - (\xi_2\, \sigma^\prime(\omega_\xi(X)))^m}
            \grad y_\xi(X) \grad y_\xi(X)^\top \right].\label{eq:Jm}
    \end{align}
    In particular, for an explicit target $\xi_2 = 0$,
    \begin{align*}
        DG_m(\xi) = \grad_\theta^2 R(\xi, \xi), \qquad G_m(\theta) - \grad_\theta R(\theta, \xi) = O(\norm{\theta - \xi}_2^{m + 1}).
    \end{align*}
    The Hessian differentiates the fitted parameter, the first argument of $R$.
\end{restatable}

The explicit-target conclusion follows from the vanishing state Jacobian, as in \citet{bolte2023one}.

\begin{proof}
    We choose a closed ball about $\xi$ on which $\abs{\theta_2} \leq \delta_2 < 1 / L$. The fixed-point equation gives
    \begin{align*}
        \abs{y_\theta(x)} \leq \frac{\abs{\sigma(0)} + L \norm{\theta_1}_2\, \norm{x}_2}{1 - L \delta_2}.
    \end{align*}
    The implicit derivative formula therefore implies, uniformly on this ball,
    \begin{align*}
        \norm*{\begin{pmatrix} x \\ y_\theta(x) \end{pmatrix}}_2
        + \norm{\grad_\theta y_\theta(x)}_2 \leq C(1 + \norm{x}_2).
    \end{align*}
    The implicit function theorem and the segment integral formula give
    \begin{align*}
        y_\theta(x) - y_\xi(x)
         & = \grad y_\xi(x)^\top(\theta - \xi)
        + o_x(\norm{\theta - \xi}_2),                    \\
        \abs{y_\theta(x) - y_\xi(x)}
         & \leq C \norm{\theta - \xi}_2(1 + \norm{x}_2).
    \end{align*}
    Using $(1 - a)\, p_m(a) = 1 - a^m$, we rewrite the backward direction as
    \begin{align*}
        G_m(\theta) = 2 \E \left[ (y_\theta(X) - y_\xi(X)) \bigl(1 - (\theta_2\, \sigma^\prime(\omega_\theta(X)))^m \bigr) \grad_\theta y_\theta(X) \right].
    \end{align*}
    The factor multiplying $y_\theta(X) - y_\xi(X)$ is continuous in $\theta$ for each $X$.
    Substituting the expansion of $y_\theta - y_\xi$, we obtain
    \begin{align*}
        G_m(\theta)
         & = 2 \E\left[
            \bigl(1 - (\xi_2\, \sigma^\prime(\omega_\xi(X)))^m\bigr)
        \grad y_\xi(X)\, \grad y_\xi(X)^\top \right](\theta - \xi) \\
         & \quad + o(\norm{\theta - \xi}_2).
    \end{align*}
    Indeed, the pointwise remainder divided by $\norm{\theta - \xi}_2$ tends to zero and is bounded by
    $C_m(1 + \norm{X}_2)^2$, which is integrable. Dominated convergence therefore justifies the remainder above and proves \eqref{eq:Jm}.
    Differentiating the risk once under the expectation gives
    $\grad_\theta R = 2 \E[(y_\theta - y_\xi)\, \grad_\theta y_\theta]$.
    Applying the same expansion without the scalar factor gives
    \begin{align*}
        \grad_\theta R(\theta, \xi)
        = 2 \E[\grad y_\xi(X)\, \grad y_\xi(X)^\top](\theta - \xi)
        + o(\norm{\theta - \xi}_2),
    \end{align*}
    proving \eqref{eq:H}. In particular, when $\xi_2 = 0$ the two derivatives agree.

    If $\xi_2 = 0$, then $\abs{\theta_2} \leq \norm{\theta - \xi}_2$, and hence
    \begin{align*}
        \norm{G_m(\theta) - \grad_\theta R(\theta, \xi)}_2
         & \leq 2L^m \abs{\theta_2}^m \E \left[
        \abs{y_\theta(X) - y_\xi(X)} \norm{\grad_\theta y_\theta(X)}_2 \right] \\
         & \leq C_m \norm{\theta - \xi}_2^{m + 1} \E[(1 + \norm{X}_2)^2].
    \end{align*}
    Since $X \in L^2(\P)$, the expectation is finite, proving the claimed order of the error.
\end{proof}

\subsection[Quantitative dependence on nonlinearity]{Quantitative bounds for \Cref{rem:nonlinearity}} \label{app:proof-nonlinearity}

We write $\Sigma_h$ for the positive-definite Gram matrix in \Cref{rem:nonlinearity} and choose $\varepsilon_0 \in (0, 1]$ with $\varepsilon_0\, \norm{h^\prime}_\infty \leq 1 / 2$. Under the assumptions of \Cref{rem:nonlinearity}, for $0 < \abs{\varepsilon} \leq \varepsilon_0$ we have
\begin{align*}
    \frac{\lambda_\mathrm{min}(\Sigma_h) \varepsilon^2}{2(d + \norm{\xi_1}_2^2 + 1)} \leq \lambda_\mathrm{min} \bigl(\grad_\theta^2 R_\varepsilon(\xi, \xi) \bigr) \leq \frac{9 \E[h(\xi_1^\top X)^2] \varepsilon^2}{2(1 + \norm{\xi_1}_2^2)}.
\end{align*}
\begin{proof}
    The matrix
    \begin{align*}
        T_\varepsilon = \begin{pmatrix} I_d & 0 \\ \xi_1^\top & \varepsilon \end{pmatrix}
    \end{align*}
    transforms $(X, h(\xi_1^\top X))$ into $(X, y_\xi(X))$. Since $1 / 2 \leq \sigma_\varepsilon^\prime \leq 3 / 2$ and $\xi_2 = 0$, formula \eqref{eq:H} gives
    \begin{align*}
        \grad_\theta^2 R_\varepsilon(\xi, \xi) \succeq \tfrac{1}{2}\, T_\varepsilon \Sigma_h T_\varepsilon^\top \succeq \tfrac{1}{2} \lambda_\mathrm{min}(\Sigma_h) T_\varepsilon T_\varepsilon^\top.
    \end{align*}
    Furthermore,
    \begin{align*}
        T_\varepsilon^{-1}
        = \begin{pmatrix} I_d & 0 \\ -\xi_1^\top / \varepsilon & 1 / \varepsilon \end{pmatrix},
        \qquad
        \norm{T_\varepsilon^{-1}}_\mathrm{op}^2
        \leq d + \frac{\norm{\xi_1}_2^2 + 1}{\varepsilon^2}
        \leq \frac{d + \norm{\xi_1}_2^2 + 1}{\varepsilon^2}.
    \end{align*}
    Since $\lambda_\mathrm{min}(T_\varepsilon T_\varepsilon^\top)
        = \norm{T_\varepsilon^{-1}}_\mathrm{op}^{-2}$, we conclude that
    \begin{align*}
        \lambda_\mathrm{min}\bigl(\grad_\theta^2 R_\varepsilon(\xi, \xi)\bigr)
        \geq \frac{\lambda_\mathrm{min}(\Sigma_h)}{2 \norm{T_\varepsilon^{-1}}_\mathrm{op}^2}
        \geq \frac{\lambda_\mathrm{min}(\Sigma_h)\, \varepsilon^2}
        {2(d + \norm{\xi_1}_2^2 + 1)}.
    \end{align*}
    For the upper bound, we take the unit vector $u = (-\xi_1, 1) / \sqrt{1 + \norm{\xi_1}_2^2}$. Then
    \begin{align*}
        u^\top \begin{pmatrix} X \\ y_\xi(X) \end{pmatrix}
        = \frac{\varepsilon h(\xi_1^\top X)}{\sqrt{1 + \norm{\xi_1}_2^2}},
    \end{align*}
    so \eqref{eq:H} and $\sigma_\varepsilon^\prime \leq 3 / 2$ give
    \begin{align*}
        \lambda_\mathrm{min} \bigl(\grad_\theta^2 R_\varepsilon(\xi, \xi) \bigr) & \leq u^\top \grad_\theta^2 R_\varepsilon(\xi, \xi)\, u \leq \frac{9 \varepsilon^2 \E[h(\xi_1^\top X)^2]}{2(1 + \norm{\xi_1}_2^2)}. \qedhere
    \end{align*}
\end{proof}

\end{appendices}

\end{document}